\documentclass{article}

\usepackage{fullpage}
\usepackage{mathtools}
\usepackage{amssymb,amsthm,bm}
\usepackage{derivative}
\usepackage{forest}
\usepackage{mleftright}
\usepackage{booktabs}
\usepackage{zref-clever}
\usepackage{algorithm,algpseudocodex}
\usepackage{float}
\usepackage{xcolor}
\usepackage{xspace}
\usepackage{authblk}
\usepackage{enumitem}
\usepackage{hyperref}
\usepackage[symbol]{footmisc}
\usepackage[natbib=true]{biblatex}
\usepackage[activate={true,nocompatibility},final,tracking=true,kerning=true,spacing=true,factor=1100,stretch=10,shrink=10]{microtype}
\usepackage{comment}

\theoremstyle{plain}

\newtheorem{lemma}{Lemma}
\newtheorem{proposition}{Proposition}
\newtheorem{corollary}{Corollary}

\theoremstyle{remark}
\newtheorem*{remark}{Remark}

\zcsetup{cap}

\definecolor{tab10blue}{RGB}{31, 119, 180} 
\definecolor{tab10orange}{RGB}{255, 127, 14}

\AtEveryBibitem{\iffieldundef{doi}{}{\clearfield{url}}}
\AtEveryBibitem{\iffieldundef{eprinttype}{}{\clearfield{url}}}

\DeclarePairedDelimiter\abs{\lvert}{\rvert}
\DeclarePairedDelimiter\norm{\lVert}{\rVert}
\DeclarePairedDelimiter\inner{\langle}{\rangle}

\DeclareMathOperator{\Tr}{Tr}

\DeclareMathOperator{\Proj}{\Pi}
\DeclareMathOperator{\op}{op}
\DeclareMathOperator{\Span}{span}

\newcommand{\GEM}{\textls[-50]{\textsc{GEM-KMeans}}\xspace}
\newcommand{\FlKM}{\textls[-50]{\textsc{Flash-KMeans}}\xspace}
\newcommand{\FPK}{\textls[-50]{\textsc{Fast-PyTorch-KMeans}}\xspace}
\newcommand{\FaKM}{\textls[-50]{\textsc{FastKMeans}}\xspace}

\DeclareMathOperator{\Sgemm}{\mathtt{Sgemm}}
\DeclareMathOperator{\Saxpy}{\mathtt{Saxpy}}
\DeclareMathOperator{\Sgemv}{\mathtt{Sgemv}}
\DeclareMathOperator{\SnrmTwo}{\mathtt{Snrm2}}
\DeclareMathOperator{\Sscal}{\mathtt{Sscal}}

\newcommand{\bOne}{\bm{1}}
\newcommand{\bZero}{\bm{0}}
\newcommand{\sfL}{\mathsf{L}}
\newcommand{\bbR}{\mathbb{R}}

\newcommand{\flagN}{\textsc{n}}
\newcommand{\flagT}{\textsc{t}}
\title{GEM-KMeans: Memory-Efficient and Accurate Clustering on Massive Scale with GPU Optimization}
\author[1]{Peng Xu$^*$}
\author[2]{Nihar Koganti$^*$}
\author[3]{Volodymyr Kindratenko}
\author[4,5,6]{Xiaohui Chen$^\dagger$}
\affil[1]{Department of Statistics, University of Illinois Urbana-Champaign \authorcr\texttt{pengxu1@illinois.edu}}
\affil[2]{Department of Electrical and Computer Engineering, University of Illinois Urbana-Champaign \authorcr\texttt{nihark2@illinois.edu}}
\affil[3]{National Center for Supercomputing Applications, University of Illinois Urbana-Champaign \authorcr\texttt{kindrtnk@illinois.edu}}
\affil[4]{Department of Mathematics, University of Southern California}
\affil[5]{Thomas Lord Department of Computer Science, University of Southern California \authorcr\texttt{xiaohuic@usc.edu}}
\affil[6]{Computing and Mathematical Sciences, California Institute of Technology \authorcr\texttt{xiaohuic@caltech.edu}}

\begin{document}

\maketitle

\footnotetext[1]{These authors contributed equally.}
\footnotetext[2]{Research partially supported by NSF DMS-2413404 and an unrestricted gift from the Simons Foundation.}

\begin{abstract}
Memory-efficient scaling on clustering problems without sacrificing statistical accuracy is of central interest for large-scale data analysis and machine learning problems. Nonnegative low-rank (NLR) matrix factorization for $K$-means is a scalable clustering method, which connects to semidefinite relaxations with optimal average-case exact recovery guarantees. However, a direct GPU implementation of NLR requires multiple large factor-sized buffers and substantial data movements that are essentially memory-bound. In this paper, we introduce \GEM, a spectrally normalized yet mathematically equivalent NLR formulation that fuses the gradient update, nonnegative projection, and sufficient statistics for normalization and iterate movement into a matrix-multiplication epilogue. Instead of retaining three massive factor-sized arrays, our IO-aware GPU implementation materializes only one single factor with small tile-reduction arrays as additional storage in the High Bandwidth Memory (HBM). We derive explicit memory costs and spectrally normalized smoothness bounds for optimizing the clustering objective function. Accurate clustering is demonstrated at massive scales on synthetic and real datasets, where performance gains of \GEM over existing GPU-accelerated Lloyd's algorithms involve data-dependent runtime tradeoffs.
\end{abstract}

\section{Introduction}
\label{sec:intro}

Clustering is a core technique in modern machine learning, underpinning a broad range of applications across data analysis, representation learning, and large-scale AI systems. For example, semantic deduplication groups related training examples before selecting data~\citep{abbas2023semdedup}; clustering groups sources with similar observed properties in massive astronomical catalogs~\citep{hunt2023openclusters}; retrieval systems organize representations around
centroids~\citep{Santhanam2022PLAID}; and clustering allows selective cache
access and sparse attention for large language model inference~\citep{Liu2025ClusterKV,zhu2025tacticadaptivesparseattention}. The $K$-means method is arguably one of the most widely used clustering methods, partially due to the fast Lloyd's algorithm iterating between nearest-centroid assignment and mean updates~\citep{macqueen1967multivariate,lloyd1982least}. Despite being scalable with highly efficient GPU implementations~\citep{johnson2021billion,yang2026flash}, this heuristic approach can be sensitive to initialization and it only enjoys suitable theoretical guarantees in recovering cluster labels when carefully initialized~\citep {ZhuangChenYang2022LASDP,qianzhangchen2022KMeans,lu2016statistical}.

Since the exact solution of the $K$-means problem, as a combinatorial optimization problem, is computationally intractable in general~\citep{dasgupta2008hardness,aloise2009np}, there are many other relaxed $K$-means formulations, including spectral clustering~\citep{von2007tutorial,ng2001spectral}, nonnegative matrix factorization (NMF)~\citep{he2011symmetric,kuang2015symnmf,wang2012nonnegative}, (convex) semidefinite programming (SDP)~\citep{peng2007approximating,mixon2017clustering,giraud2019partial}, and nonconvex low-rank SDPs~\citep{zhuang2024statistically,xu2025scalablesecondorderriemannianoptimization,CarsonMixonVillarWard_manifold-Kmeans}. Nevertheless, unlike Lloyd's algorithm, it is extremely challenging to adopt those $K$-means variants in massive-scale problems with hundreds of millions of data points, due to hardware (e.g., memory and data movement) or algorithmic (e.g., high computational complexity) constraints. This raises the following natural tradeoff question between statistical accuracy and computational efficiency:
\begin{quote}
    \emph{Can we develop a statistically accurate, or even optimal, $K$-means implementation that is hardware efficient and scalable to a massive number of data points?}
\end{quote}

In this paper, we introduce the \GEM (\emph{GEneral Matrix multiplication $K$-Means}) method, a \textbf{GPU-native}, \textbf{highly scalable}, \emph{and} \textbf{statistically exact} $K$-means approach based on the nonnegative low-rank (NLR) algorithm, originally introduced by~\citep{zhuang2024statistically}. Unlike existing GPU-accelerated Lloyd's algorithm with min-assignment and centroid update kernels \citep{yang2026flash,johnson2021billion,faiss2026clustering,rapids2026cuml}, our implementation is gradient-based that is naturally compatible with modern GPU hardware design and especially suitable for dense computes via Tensor Core kernels. With fused gradient kernels based on low-level NVIDIA libraries such as cuBLAS~\citep{nvidia2026cublas} and CUTLASS GEMM~\citep{nvidia2026cutlass}, we demonstrate that \GEM, on a single H200 NVIDIA GPU, can effectively scale to hundreds of millions of data points, outperforming state-of-the-art GPU implementations of Lloyd's algorithm on comparable scales in terms of clustering accuracy.

\subsection{Our Contributions}

Our contributions center on reducing the memory footprint to execute the full-data NLR formulation while retaining its statistical structure and demonstrating clustering accuracy at massive scale.

\begin{enumerate}[leftmargin=1.5em,labelsep=0.5em]
    \item \textbf{Memory-efficient GPU execution through fused updates and statistics.}
    We develop new customized cuBLAS/CUTLASS kernels tailored to fuse (dense) gradient updates, nonnegative projection, and accumulation of normalization and primal-movement statistics into a GEMM epilogue (\zcref{sec:gpu}). By collecting three sufficient statistics before overwriting the factor, \GEM eliminates both the objective-gradient buffer and the previous-factor copy. This reduces factor-sized storage from three $n\times r$ arrays to one, saving $2nr$ FP32 entries ($8nr$ bytes), with small tile-reduction arrays as additional storage. We give explicit buffer accounting and retain $O(ndr)$ arithmetic per iteration without subsampling the data or weakening the NLR constraints.

    \item \textbf{Normalized optimization with explicit smoothness and descent guarantees.}
    We combine spectral data normalization with a scaled constraint residual to control the gradient scale (\zcref{sec:analysis}). For a fixed dual variable, we prove the convergence of projected gradient descent for an appropriately fixed step size.
    \item \textbf{Accuracy at hundred-million-sample scale.}
    On a single NVIDIA H200 GPU, we evaluate balanced Gaussian mixtures, CyTOF cell profiles, and 118 million Gaia DR3 sources. \GEM attains near-perfect Adjusted Rand Index (ARI) in most synthetic runs (\zcref{fig:gmm}). On 104184 CyTOF cells, its median Normalized Mutual Information (NMI) is approximately $0.82$, versus $0.79$ for the three baselines (\zcref{fig:cytof_full}). On Gaia, its mean NMI is $0.315$, compared with $0.304$ for \FlKM and $0.305$ for \FPK (\zcref{fig:gaia}).
\end{enumerate}

To our knowledge, \GEM is the first $K$-means clustering solver that scales to \textbf{hundreds of millions} points with \textbf{guaranteed statistical accuracy} in the exact recovery regime under standard statistical sampling models such as Gaussian mixtures.

\subsection{Related Work}
Existing GPU accelerations for $K$-means are exclusively based on Llyod's algorithm. Using the FlashAttention mechanism~\citep{FlashAttention2022}, \FlKM reduces data movement for the Lloyd assignment and centroid update operations \citep{yang2026flash}. FAISS and cuML provide additional practical GPU baselines
\citep{johnson2021billion,faiss2026clustering,rapids2026cuml}.
$K$-means++ and scalable $K$-means++ improve initialization, while mini-batch $K$-means trades full data updates for smaller batches
\citep{arthur2007kmeanspp,bahmani2012scalable,sculley2010webscale}.

Up to a data spectral normalization, our \GEM formulation is mathematically equivalent to the NLR method of \citet{zhuang2024statistically}, whose unfused CPU implementation is limited to the order of a few tens of thousands due to memory constraints. The second-order Riemannian method of \citet{xu2025scalablesecondorderriemannianoptimization} addresses the same NLR problem with a smooth log-barrier penalty and a different optimization geometry. Our GPU implementation shares the same general principle as FlashAttention and \FlKM by consuming a matrix-product tile in its epilogue based on the fusion mechanism of CUTLASS~\citep{nvidia2026cutlass}. However, our computation has an extra projection operation which still requires a global norm reduction.

\section{\texorpdfstring{$K$}{K}-means from Partition to Nonnegative Factorization}
Let $X\in\bbR^{n\times d}$ contain $d$-dimensional observations $x_1, \dots, x_n$ as rows, and let $K$ be the true number of nonempty clusters. For any given partition $\mathcal G=(G_1,\dotsc,G_K)$ of the data index $[n]\coloneqq\{1,\dots,n\}$,
the total intra-cluster sum of squares is
\begin{equation}
 J(\mathcal G)=\sum_{k=1}^K\sum_{i\in G_k}\norm{x_i-\bar x_k}_2^2
 =\norm{X}_F^2-\inner{XX^\top, Z_{\mathcal G}},\qquad
 Z_{\mathcal G}=HH^\top,
 \label{eq:partition}
\end{equation}
where $\bar x_k$ is the $k$-th cluster mean and $H_{n \times K}$ is the associated assignment matrix with entries
$H_{ik}=\bOne\{i\in G_k\}/\sqrt{\abs{G_k}}$. Then, the partition-based formulation of the $K$-means clustering is to solve
\begin{equation}
    \label{eqn:Kmeans_partition}
    \min_{\mathcal G} J(\mathcal G)
\end{equation}
over all possible partitions $\mathcal G$ of $[n]$. Note that the $n \times n$ membership matrix $Z_{\mathcal G}$ is positive-semidefinite, entrywise nonnegative, has trace $K$, and satisfies
$Z_{\mathcal G}\bOne_n=\bOne_n$. Relaxing the partition structure that $H$ is an $n \times K$ ``binary'' assignment matrix yields the
Peng--Wei SDP \citep{peng2007approximating}:
\begin{equation}
 \min_{Z \in \bbR^{n\times n}} \mleft\{\inner{A, Z}:Z\succeq0,\; Z\geq0,\;\Tr(Z)=K,\; Z\bOne_n=\bOne_n \mright\},
 \label{eq:sdp}
\end{equation}
where $A=-XX^\top$ and $Z\geq0$ denotes entrywise nonnegativity of the matrix $Z$.

In general, the SDP relaxation~(\ref{eq:sdp}) to a continuous optimization problem may lose the integrality nature of cluster labels~\citep{FeiChen2022HiddenIntegrality}. Nevertheless, a well-known information-theoretical optimality result for the Peng--Wei SDP states that the unique SDP solution~(\ref{eq:sdp}) achieves the \emph{exact recovery} of cluster labels, as soon as this becomes possible under the standard Gaussian mixture model (GMM)~\citep{chenyang2021threshold}. Specifically, consider the standard GMM $x_i=\mu_{z_i^*}+\varepsilon_i$ with i.i.d. sampling noises
$\varepsilon_i\sim N(0,\sigma^2I_d)$, equal cluster sizes, and the minimal centroid separation
$\Delta=\min_{k\ne\ell \in [K]}\norm{\mu_k-\mu_\ell}_2$. Here, $z_i^* = k$ if $i \in G_k^*$ means that the data point $x_i$ belongs to the true $k$-th cluster $G_k^*$. Then, there is a phase transition on the \emph{centroid-to-noise ratio} (CNR), with the sharp threshold given by
\begin{equation}
    \label{eqn:exact_recovery_threshold}
    \overline{\text{CNR}}^2\coloneqq 4 \left( 1+ \sqrt{1+{\frac{Kd}{n \log{n}}}} \right) \log{n},
\end{equation}
in the sense that: for any $\alpha > 0$ and $K = O(\log{n} / \log\log{n})$, (i) when the minimal centroid-separation $\Delta \geq (1+\alpha) \sigma \cdot \overline{\text{CNR}}$, the SDP solution of~(\ref{eq:sdp}) has a block diagonal structure (after data rearrangement), which perfectly recovers the true cluster structure ${\mathcal G}^* = (G_1^*, \dotsc, G_K^*)$ with high probability; (ii) when $\Delta \leq (1-\alpha) \sigma \cdot \overline{\text{CNR}}$, then the maximum likelihood training, and therefore any other method regardless of computational cost, fails to exactly recover the true cluster structure.

Despite the appealing statistical optimality and robust empirical performance, SDP is often criticized for having poor scalability. Even for a small dataset with sample size $n = 1000$, the SDP requires $n^2 = 10^6$ optimization variables that are at the very boundary of state-of-the-art general-purpose SDP solvers such as SDPNAL+~\citep{yang2015sdpnal+}. In the clustering problem when the SDP solution is expected to be low-rank, nonnegative low-rank (NLR) matrix factorization is a computationally thrifty nonconvex approach by reparametrizing $Z = U U^T$, where $U \geq 0$ is an $n \times r$ nonnegative matrix factor with a rank parameter $r \geq K$~\citep{burer2003nonlinear}. This substantially reduces the number of variables from $n^2$ to $n r$ at the cost of forfeiting the SDP convexity. Specifically, for the $K$-means SDP~(\ref{eq:sdp}), ~\cite{zhuang2024statistically} proposed to solve the following nonconvex NLR problem:
\begin{equation}
    \label{eqn:Kmeans_NLR}
    \min_{U \in \Omega} \mleft\{ \inner{A, U U^\top}: g(U) = \bZero_n \mright\},
\end{equation}
where $\Omega=\{U\in\bbR^{n\times r}:U\geq0,\norm{U}_F^2=K\}$ and $g(U) = U U^T \bOne_n -\bOne_n$. Note that problem (\ref{eqn:Kmeans_NLR}) is a tighter NMF-style formulation than the SDP (\ref{eq:sdp}) because it imposes the entrywise nonnegative constraint $U \geq 0$ (rather than $U U^T \geq 0$) that is easier to enforce~\citep{kulis2007fast}. Tailoring the standard primal-dual algorithm (cf.~\zcref{alg:vanilla_NLR} in Appendix) via the augmented Lagrangian method~\citep{NocedalWright2006_NumOpt} to minimize
\begin{equation}
     \label{eq:ALM_primal_dual}
     \min_{U \in \Omega} \mleft\{ f_y(U)\coloneqq\inner{A, UU^\top}+y^\top g(U)+\frac{\beta}{2}\norm{g(U)}_2^2 \mright\},
\end{equation}
it is shown in~\citep{zhuang2024statistically} that with a random initialization, the NLR algorithm enjoys the local linear convergence to a stationary point in the exact recovery regime $\Delta \geq (1+\alpha) \sigma \cdot \overline{\text{CNR}}$.

\section{Scaling Issues: Optimization Stability and Massive Materialization}
\label{sec:analysis}
Scaling NLR to massive datasets presents two related challenges: keeping the gradient well-scaled and avoiding unnecessary transfers of $n\times r$ matrices. For the objective $f_y(U)$, the penalty gradient contains $(\bOne_n \overline{y}^\top + \overline{y} \bOne_n^\top) U$ with $\overline{y} = y + \beta(U U^\top \bOne_n - \bOne_n)$, whose Frobenius norm increases with $n$. This requires the penalty coefficient or primal step size to be adjusted as the sample size changes. To make the step-size adjustment in our optimization more scale-stable with respect to large sample size $n$, we therefore center and spectrally normalize the data:
\[
X_c=X-\bOne_n\bar x^\top,\qquad
a=\norm{X_c}_2>0,\qquad
\hat X=X_c/a,\qquad
\hat A=-\hat X\hat X^\top.
\]
Accordingly, we also normalize the constraint residual and define
\begin{equation}
\hat f_y(U)
=\inner{\hat A,UU^\top}+y^\top h(U)
 +\frac{\beta}{2}\norm{h(U)}_2^2,
\qquad
h(U)=\frac{g(U)}{\sqrt n}.
\label{eq:normalized}
\end{equation}
Here $y$ is the multiplier for $h$, rather than $g$. The resulting gradient is
\begin{equation}
\nabla_U\hat f_y(U)
=2\hat A U
+\bar y_{\mathrm{eff}}(U^\top\bOne_n)^\top
+\bOne_n(U^\top\bar y_{\mathrm{eff}})^\top,
\label{eq:gradient}
\end{equation}
where the effective multiplier and the dual ascent update are
\begin{equation}
\bar y_{\mathrm{eff}}
=\frac{y}{\sqrt n}+\frac{\beta}{n}g(U),
\qquad
y^+=y+\frac{\beta}{\sqrt n}g(U).
\label{eq:dualnormalized}
\end{equation}
The implementation stores $y/\sqrt n$ and uses $\beta/n$ as its penalty coefficient, so its dual buffer is updated by $(\beta/n)g(U)$. All formulas below retain $y$ as the multiplier for $h$. We remark that centering and positive rescaling preserve the optimal clustering partitions, although they can change the path of infeasible iterates.

Our spectral normalization gives a smoothness bound for each fixed value of the dual variable. This bound supplies a sufficient step size for the projected primal update and yields the following descent result.
\begin{proposition}[A global fixed-dual smoothness bound]
\label{prop:smooth}
For fixed $y$ and $\beta>0$, the gradient of (\ref{eq:normalized}) is Lipschitz on $\{U:\norm{U}_F\leq\sqrt K\}$, with the valid bound
\begin{equation}
\sfL=2+2\norm{y}_2+\beta(6K+2).
\label{eq:L}
\end{equation}
\end{proposition}
\begin{corollary}
\label{cor:descent}
If $U\in\Omega$, $U^+\in\Proj_\Omega\bigl(U-\alpha\nabla\widehat f_y(U)\bigr)$ is an exact Euclidean projection and $0<\alpha<1/\sfL$, then
\begin{equation}
\widehat f_y(U^+)\leq\widehat f_y(U)-\frac{1-\alpha \sfL}{2\alpha}\norm{U^+-U}_F^2.
\label{eq:descent}
\end{equation}
\end{corollary}
As a direct consequence by summing over (\ref{eq:descent}) and using the boundedness below of $\hat f_y$ on $\Omega$, we have for fixed $y$ and a constant step size $0<\alpha<1/\sfL$, the projected gradient iterates satisfy 
\[
\sum_{t=0}^{\infty}\norm{U_{t+1}-U_t}_F^2<\infty,
\qquad
\norm{U_{t+1}-U_t}_F\to0.
\]

The second scaling challenge is the repeated memory transfers in the primal projected gradient descent (PGD) updating steps. Specifically, for each fixed dual variable $y$, the primal PGD iteration computes
\begin{equation}
U^+\in\Proj_\Omega\bigl(U-\alpha\nabla_U\hat f_y(U)\bigr),
\qquad
\Proj_\Omega(W)=\sqrt K\,\frac{W_+}{\norm{W_+}_F},
\qquad W_+=\max(W,0),
\label{eq:projected-primal-step}
\end{equation}
when $W_+\ne0$. Then, the NLR algorithm checks the constraint residual in the augmented-Lagrangian term $h(U)=(UU^\top\bOne_n-\bOne_n) / \sqrt{n}$ and applies the dual update in \eqref{eq:dualnormalized} when relative primal movement is sufficiently small. A straightforward implementation would materialize the whole $n\times r$ gradient and projected update, then read successive iterates again to measure movement. These intermediate matrices and repeated memory transfers become prohibitively expensive and memory-bound at very large $n$. To address this issue, we develop a memory-efficient GPU implementation of the primal-dual NLR iterations at scale.

\section{Our Method: Efficient GPU Implementation}
\label{sec:gpu}

Since our method \GEM is a memory-efficient GPU implementation of the $K$-means NLR algorithm and they are mathematically equivalent between (\ref{eq:ALM_primal_dual}) and (\ref{eq:normalized}), we first list the key updating formula, where a naive CPU/GPU implementation requires materialization of the full gradient with memory-bounded matrix multiplications. Then, we propose a new IO-aware implementation with fused gradient kernels based on the cuBLAS and CUTLASS GEMM libraries, which combine the entrywise update with the matrix multiplication used to evaluate that product. A key component in our gradient kernels is to accumulate the statistics needed for online projection and data movement monitoring before overwriting the current factor. This eliminates both the full objective-gradient buffer and a second copy of the factor.

\subsection{GEM-KMeans updating formula}
\label{sec:gpu-update}

Define
\begin{equation}
 p=U^\top\bOne_n,\qquad
 q=Up-\bOne_n=g(U),\qquad
 \bar y=\bar y_{\text{eff}}
 =\frac{y}{\sqrt n}+\frac{\beta}{n}q.
 \label{eq:gpu-vectors}
\end{equation}
Using the gradient in (\ref{eq:gradient}) with $\Theta=\hat X^\top U$ and $s=U^\top\bar y$, the positive part of the tentative update is
\begin{equation}
 W=\mleft[
 U+2\alpha\hat X\Theta
 -\alpha\bigl(\bar y p^\top+\bOne_n s^\top\bigr)
 \mright]_+.
 \label{eq:gpu-positive-update}
\end{equation}
Note that the Euclidean projection in
(\ref{eq:projected-primal-step}) is given by $U^+=\sqrt K W / \norm{W}_F$ when $\norm{W}_F>0$, where
the positive-part operation is entrywise and can be fused with the gradient computation. The normalization depends on all entries and follows a global reduction. Associating the objective-gradient product as $\hat X(\hat X^\top U)$ avoids materializing
$\hat X\hat X^\top$. The two constraint-gradient terms are generated from vectors rather than stored as $n\times r$ matrices.

\subsection{Fused gradient update and statistics}
\label{sec:fusion}

We first compute $\Theta=\hat X^\top U$, using cuBLAS. The second product, $\hat X\Theta$, is computed by a CUTLASS GEMM whose epilogue applies the entrywise update and accumulates the statistics needed for normalization and convergence criterion.
The kernel configuration is described in \zcref[S]{app:gpuops}.

For each valid output entry $(i,j)$, let $a_{ij}$ denote the GEMM accumulator for $(\hat X\Theta)_{ij}$ and $w_{ij}$ denote the entry of the nonnegative update $W$ in (\ref{eq:gpu-positive-update}) before normalization:
\begin{equation}
 w_{ij}=\max\mleft\{
 U_{ij}+2\alpha a_{ij}
 -\alpha\bar y_i p_j-\alpha s_j,\;0
 \mright\}.
 \label{eq:entrywise-fusion}
\end{equation}
The epilogue combines the GEMM accumulator with the old factor entry and the vectors \(\bar y,p,s\), producing \(w_{ij}\) without storing intermediate gradient matrices of size $n \times r$. Once the normalization multiplier \(c\) is available, the squared movement follows from
\begin{equation}
 \norm{cW-U_{\text{prev}}}_F^2
 =c^2\norm{W}_F^2+\norm{U_{\text{prev}}}_F^2
  -2c\inner{W,U_{\text{prev}}}_F.
 \label{eq:movement-identity}
\end{equation}

Before overwriting the old factor entries, the epilogue
accumulates partial sums for
\begin{equation}
 a=\norm{W}_F^2,\qquad
 b=\norm{U_{\text{prev}}}_F^2,\qquad
 h=\inner{W,U_{\text{prev}}}_F.
 \label{eq:fused-statistics}
\end{equation}
Here, $a$ is the squared Frobenius norm of the unnormalized
update $W$, $b$ is the squared Frobenius norm of the previous
factor, and $h$ is their Frobenius inner product.
The epilogue accumulates these statistics in FP64 before overwriting the old factor entries with $W$. This allows us to estimate primal movement without retaining a second copy of $U$. Finally, the partial sums per tile are reduced to the global scalars $a,b,h$ and are used to normalize the factor using $U\gets cW$, where $c=\operatorname{FP32}(\sqrt{K/a})$.
This scaling pass follows the global reduction because $c$ depends on all output tiles.

\subsection{Primal movement and stopping}
\label{sec:movement}

Using the accumulated statistics, we estimate relative primal movement as
\begin{equation}
 \hat\delta
 =\sqrt{\frac{\max\{c^2a+b-2ch,0\}}{K}}.
 \label{eq:movement-estimate}
\end{equation}
The clamp removes negative values caused by floating-point cancellation. The estimate uses the rounded multiplier $c$, but does not account for entrywise rounding when $cW$ is stored.

We use the accumulated squared norms $a$ and $b$ rather than replacing $c^2a+b$ with $2K$.
Floating-point normalization does not enforce the target squared norm $K$ exactly, and this difference can affect the movement estimate near convergence.

After scaling, we recompute $p=(U^+)^\top\bOne_n$ and $q=U^+p-\bOne_n$. When $\hat\delta<\epsilon_d$, we take the dual step $y\gets y+(\beta/\sqrt n)q$. We terminate when
\[
 \hat\delta<\min\{\epsilon_d,\epsilon_c\},
 \qquad
 \frac{\norm{q}_2}{\sqrt n}<\epsilon_c,
 \qquad
 t+1\geq T_{\min}.
\]
The iteration limit is $T$. At each iteration, primal movement determines whether
to update the dual variable, while both movement and constraint feasibility determine whether to terminate.

The dual ascent tolerance $\epsilon_d$ and convergence tolerance $\epsilon_c$ are configurable parameters and could vary based on the clustering problem.

\begin{algorithm}[t]
\caption{GEM iteration with fused update and statistics}
\label{alg:gem}
\begin{algorithmic}[1]
\Require $\hat X$, $U\in\Omega$, $y=0$, $\alpha,\beta>0$, $T$, $T_{\min}$, $\epsilon_d$, $\epsilon_c$.
\State $p\gets U^\top\bOne$; $q\gets Up-\bOne$;
\For{$t=0,\dotsc,T-1$}
    \State $\Theta\gets\hat X^\top U$;
    \State $\bar y\gets y/\sqrt n+(\beta/n)q$;
           $s\gets U^\top\bar y$;
    \State Compute $W$ by (\ref{eq:gpu-positive-update})
           and tile partials of $a,b,h$ in one GEMM epilogue;
    \State Reduce partials and transfer $a,b,h$ to the host;
    \If{$\sqrt a\leq10^{-20}$}
      \State Report normalization breakdown and \textbf{break};
    \EndIf
    \State $c\gets\operatorname{float}(\sqrt{K/a})$;
    \State $U\gets cW$;
    \State Compute $\hat\delta$ by
           (\ref{eq:movement-estimate});
    \State $p\gets U^\top\bOne$; $q\gets Up-\bOne$;
    \If{$\hat\delta<\epsilon_d$}
      \State $y\gets y+(\beta/\sqrt n)q$;
    \EndIf
    \State $\rho\gets\norm{q}_2/\sqrt n$;
    \If{$\hat\delta<\min\{\epsilon_d,\epsilon_c\}$
         and $\rho<\epsilon_c$ and $t+1\geq T_{\min}$}
      \State \textbf{break};
    \EndIf
\EndFor
\State \textbf{return} $(U,y)$.
\end{algorithmic}
\end{algorithm}

\subsection{Arithmetic, storage, and synchronization}
\label{sec:cost}

The two matrix products require approximately $4ndr$ floating-point operations per primal iteration, counting a multiply and an add separately. The matrix-vector products, entrywise updates, and reductions add $O(nr)$ work. The total is therefore $4ndr+O(nr)$ per iteration. These counts describe useful arithmetic per iteration and exclude work on padded tile entries, preprocessing, and initialization.  For output tiles of size $B_M\times B_N$, the number
of tiles is
\begin{equation}
 N_{\text{tile}}
 =\left\lceil\frac{n}{B_M}\right\rceil
  \left\lceil\frac{r}{B_N}\right\rceil.
 \label{eq:gpu-tile-count}
\end{equation}
The reduction workspace stores three FP64 partial sums per GEMM output tile, requiring $24N_{\text{tile}}$ bytes, where $N_{\text{tile}}=\lceil n/B_M\rceil\lceil r/B_N\rceil$.
Here, $B_M$ and $B_N$ are the output tile dimensions.

\zcref[S]{tab:memory-settings} lists the device buffers used during iterations. Their total storage is
\begin{equation}
 b_f(nd+nr+dr+4n+2r)
 +3b_s(N_{\text{tile}}+1),
 \label{eq:gpu-storage}
\end{equation}
where $b_f=4$ bytes for FP32 entries and $b_s=8$ bytes for FP64 statistics. 

Since $W$ overwrites $U$, the iteration requires only one $n\times r$ factor buffer.
Fusing the gradient update and statistics into the GEMM epilogue removes the need to store the objective-gradient matrix and a separate copy of the previous factor, saving two $n\times r$ buffers, or $8nr$ bytes in FP32.

The host waits for the reduced statistics before computing the normalization multiplier, and for the feasibility norm before checking termination.

\begin{table}[t]
    \centering
    \begin{minipage}[t]{0.46\textwidth}
        \centering
        \caption{Device buffers used during the iteration, excluding library workspaces and preprocessing temporaries.}
        \label{tab:memory-settings}
        \begin{tabular}[t]{@{}llr@{}}
        \toprule
        Buffer & Type & Elements\\
        \midrule
        Input $\hat X$ & FP32 & $nd$\\
        Factor $U$, overwritten by $W$ & FP32 & $nr$\\
        Product $\Theta$ & FP32 & $dr$\\
        $\bOne,q,y/\sqrt n,\bar y$ & FP32 & $4n$\\
        $p,s$ & FP32 & $2r$\\
        Three tile-statistic arrays & FP64 & $3N_{\text{tile}}$\\
        Three reduced statistics & FP64 & $3$\\
        \bottomrule
        \end{tabular}
    \end{minipage}
    \hfill
    \begin{minipage}[t]{0.46\textwidth}
        \centering
        \caption{Experimental settings in \zcref{subsec:sim_bench} with the NLR factorization rank $r$.\vspace{\baselineskip}}\label{tab:gmm_regime}
        \begin{tabular}[t]{@{}rrrrr@{}}
        \toprule
        $n$ & $d$ & $K$ & $r$ & Seeds\\
        \midrule
        $10^6$            & 50  & 20  & 30  & 20\\
        $10^6$            & 200 & 100 & 300 & 5\\
        $10^7$            & 50  & 20  & 30  & 20\\
        $10^7$            & 200 & 100 & 300 & 5\\
        $10^8$            & 50  & 20  & 30  & 5\\
        $2\times 10^8$    & 50  & 20  & 60  & 5 \\
        \bottomrule
        \end{tabular}
    \end{minipage}
    \vspace{-0.2in}
\end{table}

\textbf{Relation to the FlashAttention mechanism.} FlashAttention~\citep{FlashAttention2022} computes attention in tiles without storing the full attention matrix in global memory. Our \GEM follows the same principle of consuming matrix-product intermediates on chip. The CUTLASS epilogue uses the GEMM accumulators directly to update the factor and collect statistics, avoiding a separate objective-gradient buffer. Its projection depends, however, on the global norm $\norm{W}_F$. Moreover, FlashAttention updates its softmax normalization incrementally as tiles are processed. In \GEM, the normalization multiplier also depends on the global Frobenius norm of $W$, so our implementation reduces the tile statistics before applying a separate scaling pass.

\section{Experiments}
We compare \GEM (GEM hereafter) with three GPU implementations of Lloyd's algorithm: \FlKM~\citep{yang2026flash}, \FPK~\citep{omer2020fastpytorchkmeans}, and \FaKM~\citep{fastkmeans2025}. All baselines use one random initialization per seed and a stopping tolerance of $10^{-4}$. These tolerances are not directly equivalent because the implementations aggregate centroid movements differently; in particular, \FPK uses squared movements. \FaKM is additionally capped at 2000 iterations. We evaluate clustering quality against available labels and total runtime on a synthetic GMM dataset and two real-world datasets: mass cytometry (CyTOF) cell profiles and Gaia stellar observations. We report ARI for balanced datasets and NMI otherwise. All experiments run on a single NVIDIA H200 GPU; GEM uses a CUTLASS SM80 Tensor Core kernel with the customized fused epilogue described in \zcref{sec:gpu}.

\subsection{Simulation benchmarks}\label{subsec:sim_bench}
We consider three synthetic data regimes: large $n$ with small $K$, large $n$ with large $K$, and moderate $n$ with large $K$. We investigate two questions: (i) does solving the NLR relaxation recover the planted partition more reliably than Lloyd-type heuristics, and (ii) what does this cost in computation time as $n$ grows to $\sim10^8$? In each setup, the dataset is generated from a balanced $K$-component GMM in $\bbR^d$ with centroid separation $\Delta=\gamma\overline{\text{CNR}}$, $\gamma=1.0$. The settings are listed in \zcref{tab:gmm_regime}. For each seed, we generate a new dataset and use it for all methods. \zcref[S]{fig:gmm} shows the comparison results. GEM attains ARI close to one (i.e., exact recovery) in nearly every run, whereas the $K$-means baselines generally show lower accuracy and greater variation across seeds. A few isolated GEM runs fall below its typical performance, so the result is not uniform across initializations. \zcref[S]{fig:gmm} also illustrates how the runtime--accuracy trade-off changes with scale. \FPK is the fastest method in every displayed configuration. \GEM incurs greater computational cost to optimize a nonconvex factorization of the SDP relaxation, but achieves higher clustering accuracy. At $n=2\times10^8$, GPU memory becomes a limiting factor: \FaKM runs out of memory, while \FlKM requires CPU--GPU block-streaming, whose transfer overhead makes it slower than \GEM.

\begin{figure}[!htbp]
\centering
\includegraphics[width=0.95\linewidth]{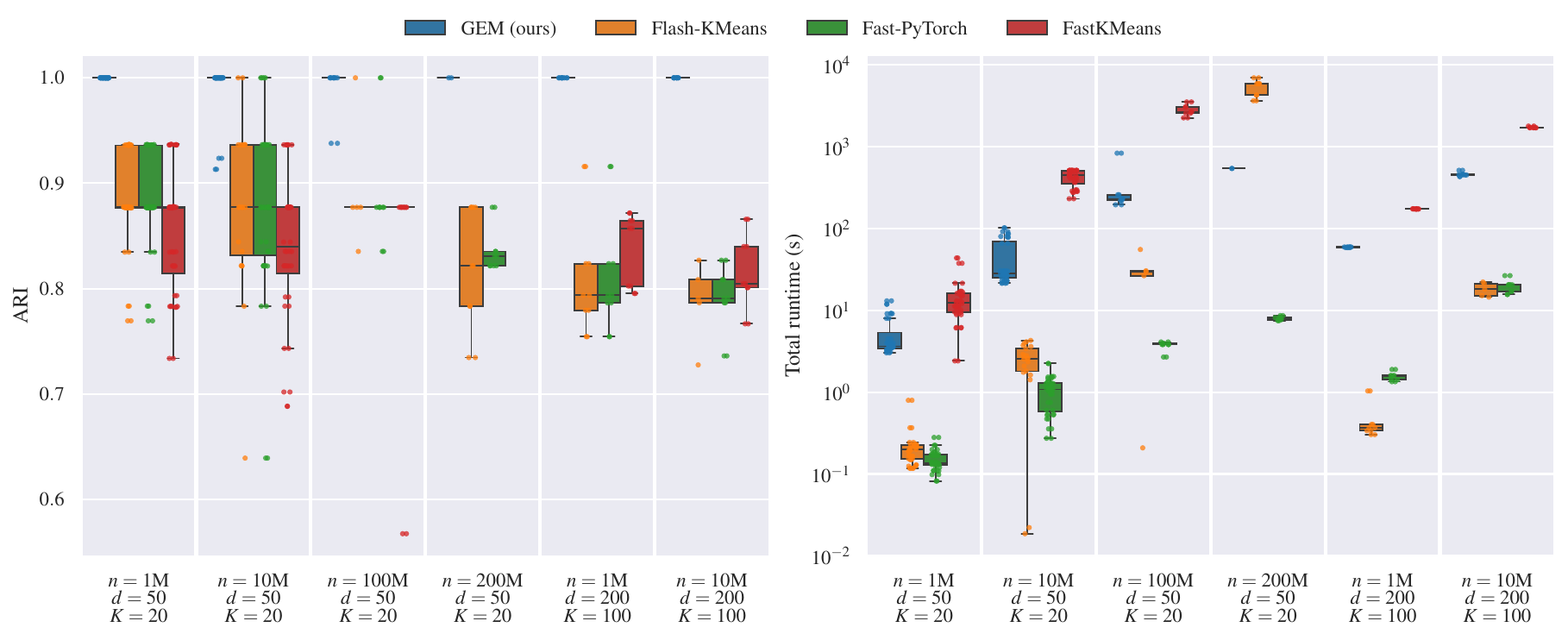}
\caption{ARI and total runtime for five balanced GMM configurations. Each colored dot represents one run; boxes show the median and interquartile range across seeds. GEM achieves consistently high ARI, with occasional lower-scoring runs.}
\label{fig:gmm}
\end{figure}

\textbf{Recovery versus runtime.} We compare GEM with \FlKM on GMMs with $K=20$, $d=50$, and different $(n,\gamma)$ combinations. \FlKM is run on 50 independently seeded datasets; on 10 matched datasets, we evaluate GEM with iteration caps $T\in\{100,300,1000,3000,10000,30000\}$, allowing runs to stop earlier upon convergence. The results shown in \zcref{fig:gmm_frontier}. Across all three settings, \FlKM rarely reaches the statistical floor: it succeeds in only 2--3 times and most runs instead finish with misclustering error between $0.07$ and $0.3$. Notably, doubling the separation to $\gamma=2$ at $n=10^6$ leaves this failure pattern largely intact, while GEM reaches the floor on all 10 matched datasets within a few seconds of iteration time. GEM also recovers on all evaluated seeds at $n=10^7,\gamma=1$, although there the required time rises to roughly $10^2$ seconds. These results favor GEM when reliable recovery matters, with a substantial runtime cost in the largest setting.

\begin{figure}[!htbp]
\centering
\includegraphics[width=0.95\linewidth]{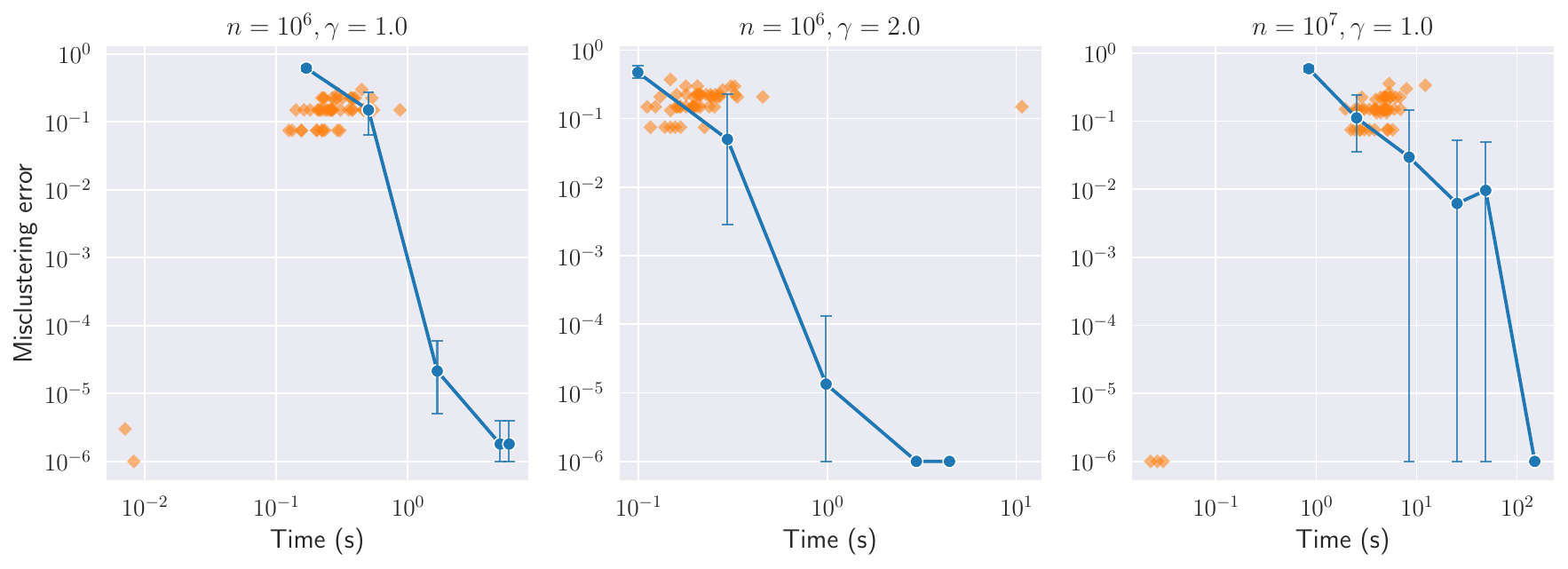}
\caption{Misclustering error versus time for balanced Gaussian mixtures under three $(n,\gamma)$ settings. \textcolor{tab10orange}{Diamonds} show 50 individual \FlKM runs. The \textcolor{tab10blue}{line} shows GEM's mean across 10 matched seeds at each iteration budget; bars span the range of error. Both axes are logarithmic. Times measure the GEM iteration loop and \FlKM wall time, respectively.}
\label{fig:gmm_frontier}
\end{figure}

\subsection{Real data applications}

\subsubsection{CyTOF}
We evaluate the methods on a real-world mass cytometry (CyTOF) dataset\citep{LEVINE2015184, CyTOFClean} containing 104,184 labeled cell profiles measured across 32 protein markers. The labels identify 14 gated cell populations with markedly unequal sizes: the largest contains 26,366 cells, whereas the smallest contains 304. \zcref[S]{fig:cytof_full} compares NMI and total runtime over 100 seeds. GEM achieves the highest median NMI, approximately 0.82, compared with approximately 0.79 for each baseline. \FPK has the shortest runtime due to its lower convergence criteria; GEM takes roughly 3 seconds per run.

\begin{figure}[!htbp]
\centering
\includegraphics[width=0.95\linewidth]{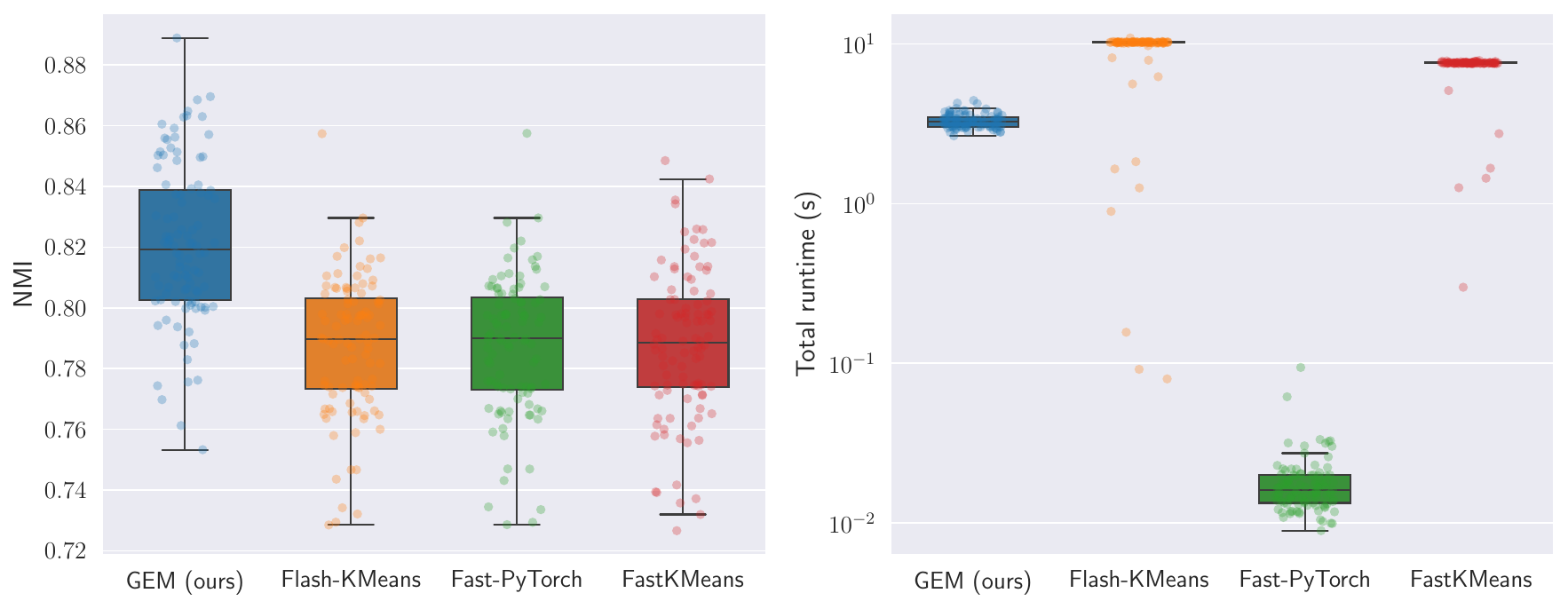}
\caption{Clustering performance on the full CyTOF dataset across 100 random seeds. (Left) NMI against the gated cell labels; (Right) total runtime on a logarithmic scale. Boxes indicate the median and interquartile range, and dots represent individual runs. GEM achieves the highest median NMI.}
\label{fig:cytof_full}
\end{figure}

\subsubsection{Gaia DR3}

\begin{figure}[!htbp]
    \centering
    \includegraphics[width=0.96\linewidth]{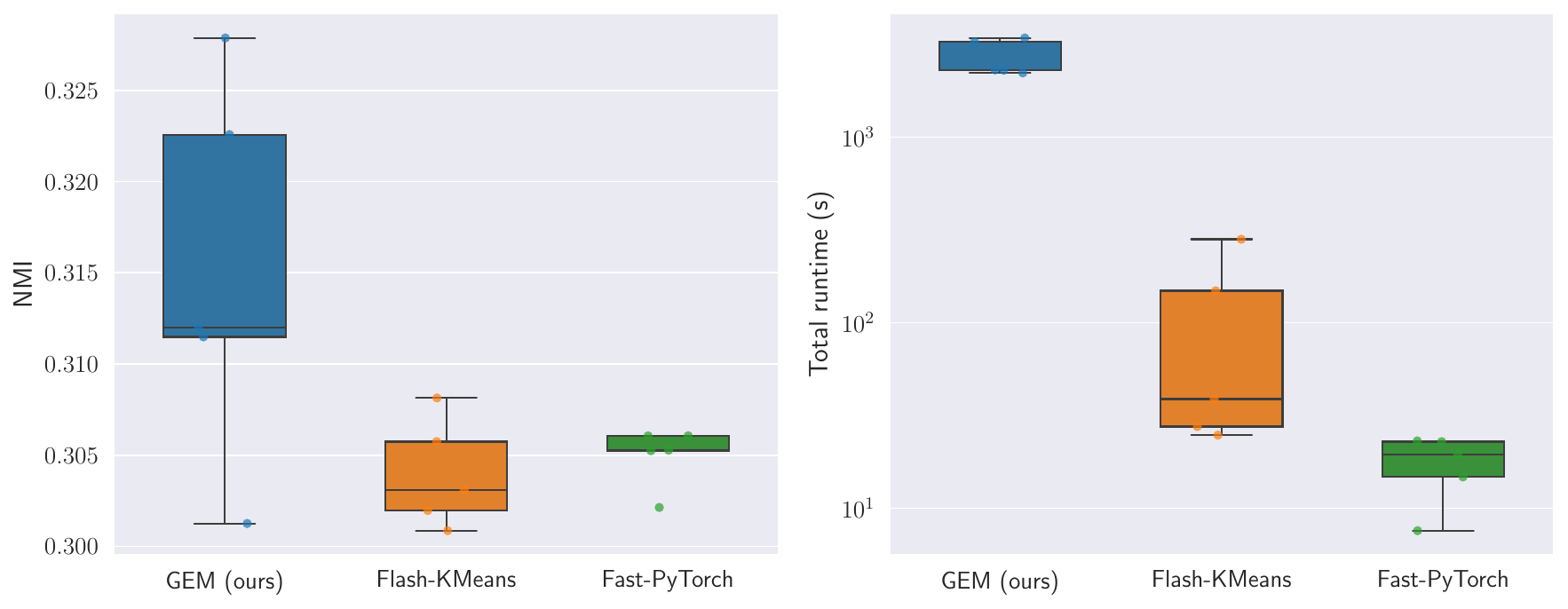}
    \caption{Clustering performance on 118 million Gaia DR3 sources across five random seeds. (Left) NMI against the seven reference spectral classes. (Right) total runtime on a logarithmic scale. Boxes show the interquartile range, horizontal lines indicate medians, and dots represent individual runs.}
    \label{fig:gaia}
\end{figure}

Gaia Data Release 3 (DR3) \citep{GaiaDR3} is the third data release of European Space Agency's Gaia space mission, and provides astrometry and photometry for about 1.8 billion sources. It also includes low-resolution BP/RP spectra for about 219 million sources, each represented by 110 Gauss--Hermite coefficients \citep{GaiaXP}. We choose 118,387,161 sources with reliable astrometry and a confident spectral class from the ESP--HS pipeline \citep{GaiaPara}. Each band is normalized to a unit $L^2$ norm. The total data size is roughly 53 GB in single precision. The seven spectral classes (O, B, A, F, G, K, M) serve as reference labels for $K = 7$. Although the geometry of this dataset is challenging for Euclidean $K$-means clustering, its scale makes it a useful benchmark for computational scalability; our aim is to evaluate clustering algorithms at this scale rather than to develop a new method for stellar classification. \zcref[S]{fig:gaia} compares clustering quality and total runtime over five random seeds. GEM achieves the highest mean and median NMI, with a mean of $0.315$, compared with $0.304$ for \FlKM and $0.305$ for \FPK. These results demonstrate that GEM can improve agreement with the reference spectral classes even at the scale of 118 million sources. This improvement comes at a higher computational cost: GEM's median runtime is approximately 38 minutes, compared with 39 seconds for \FlKM and 20 seconds for \FPK.

\newpage
\printbibliography

@misc{nvidia2026cublas,
	author = {{NVIDIA}},
	title = {{cuBLAS} Library Documentation},
	year = {2026},
	howpublished = {CUDA Toolkit documentation},
	url = {https://docs.nvidia.com/cuda/cublas/index.html},
	note = {Accessed September 25, 2026}
}

@misc{nvidia2026cutlass,
	author = {{NVIDIA}},
	title = {{CUTLASS} 3.0 {GEMM} {API}: Collective Epilogue},
	year = {2026},
	howpublished = {CUTLASS documentation},
	url = {https://docs.nvidia.com/cutlass/4.5.2/media/docs/cpp/gemm_api_3x.html},
	note = {Accessed September 25, 2026}
}

@inproceedings{arthur2007kmeanspp,
	author = {Arthur, David and Vassilvitskii, Sergei},
	title = {{k-means++}: The Advantages of Careful Seeding},
	booktitle = {Proceedings of the Eighteenth Annual ACM-SIAM Symposium on Discrete Algorithms},
	pages = {1027--1035},
	year = {2007},
	url = {https://theory.stanford.edu/~sergei/papers/kMeansPP-soda.pdf}
}

@article{bahmani2012scalable,
	author = {Bahmani, Bahman and Moseley, Benjamin and Vattani, Andrea and Kumar, Ravi and Vassilvitskii, Sergei},
	title = {Scalable {K-Means++}},
	journal = {Proceedings of the VLDB Endowment},
	volume = {5},
	number = {7},
	pages = {622--633},
	year = {2012},
	doi = {10.14778/2180912.2180915},
	url = {https://arxiv.org/abs/1203.6402}
}

@inproceedings{sculley2010webscale,
	author = {Sculley, D.},
	title = {Web-Scale {K-Means} Clustering},
	booktitle = {Proceedings of the 19th International Conference on World Wide Web},
	pages = {1177--1178},
	year = {2010},
	doi = {10.1145/1772690.1772862},
	url = {https://doi.org/10.1145/1772690.1772862}
}

@misc{rapids2026cuml,
	author = {{NVIDIA}},
	title = {{cuML KMeans} Documentation},
	year = {2026},
	url = {https://docs.nvidia.com/cuml/latest/api/generated/cuml.cluster.KMeans/},
	note = {Accessed September 25, 2026}
}

@misc{faiss2026clustering,
	author = {{Faiss developers}},
	title = {{GPU k-means} Example},
	year = {2026},
	url = {https://github.com/facebookresearch/faiss/wiki/GPU-k-means-example},
	note = {Online software documentation; accessed September 25, 2026}
}

@article{johnson2021billion,
	author = {Johnson, Jeff and Douze, Matthijs and J{\'e}gou, Herv{\'e}},
	title = {Billion-Scale Similarity Search with {GPUs}},
	journal = {IEEE Transactions on Big Data},
	volume = {7},
	number = {3},
	pages = {535--547},
	year = {2021},
	doi = {10.1109/TBDATA.2019.2921572},
	url = {https://doi.org/10.1109/TBDATA.2019.2921572}
}

@inproceedings{Santhanam2022PLAID,
	author = {Santhanam, Keshav and Khattab, Omar and Potts, Christopher and Zaharia, Matei},
	title = {{PLAID}: An Efficient Engine for Late Interaction Retrieval},
	year = {2022},
	isbn = {9781450392365},
	publisher = {Association for Computing Machinery},
	address = {New York, NY, USA},
	url = {https://doi.org/10.1145/3511808.3557325},
	doi = {10.1145/3511808.3557325},
	booktitle = {Proceedings of the 31st ACM International Conference on Information \& Knowledge Management},
	pages = {1747--1756},
	numpages = {10},
	location = {Atlanta, GA, USA},
	series = {CIKM '22}
}

@inproceedings{Liu2025ClusterKV,
	author = {Liu, Guangda and Li, Chengwei and Zhao, Jieru and Zhang, Chenqi and Guo, Minyi},
	title = {{ClusterKV}: Manipulating {LLM KV} Cache in Semantic Space for Recallable Compression},
	year = {2025},
	publisher = {IEEE Press},
	url = {https://doi.org/10.1109/DAC63849.2025.11132479},
	doi = {10.1109/DAC63849.2025.11132479},
	booktitle = {Proceedings of the 62nd Annual ACM/IEEE Design Automation Conference},
	numpages = {7},
	location = {San Francisco, California, United States},
	series = {DAC '25},
	pages = {1--7}
}

@inproceedings{zhu2025tacticadaptivesparseattention,
	title = {{Tactic: Adaptive Sparse Attention with Clustering and Distribution Fitting for Long-Context LLMs}},
	author = {Zhu, Kan and Tang, Tian and Xu, Qinyu and Jin, Zhan and Gu, Yile and Zeng, Zhichen and Kadekodi, Rohan and Zhao, Liangyu and Li, Ang and Krishnamurthy, Arvind and Kasikci, Baris},
	year = {2026},
	url = {https://proceedings.iclr.cc/paper_files/paper/2026/hash/33f94d79acf71051d6a27f4d8889e20e-Abstract-Conference.html},
	booktitle = {International Conference on Learning Representations}
}

@inproceedings{FlashAttention2022,
	author = {Dao, Tri and Fu, Daniel Y. and Ermon, Stefano and Rudra, Atri and R{\'e}, Christopher},
	title = {{FlashAttention}: Fast and Memory-Efficient Exact Attention with {IO}-Awareness},
	year = {2022},
	isbn = {9781713871088},
	publisher = {Curran Associates Inc.},
	address = {Red Hook, NY, USA},
	booktitle = {Advances in Neural Information Processing Systems},
	articleno = {1189},
	numpages = {16},
	location = {New Orleans, LA, USA},
	volume = {35},
	pages = {16344--16359},
	doi = {10.52202/068431-1189},
	url = {https://proceedings.neurips.cc/paper_files/paper/2022/hash/67d57c32e20fd0a7a302cb81d36e40d5-Abstract-Conference.html}
}

@article{FeiChen2022HiddenIntegrality,
	author = {Yingjie Fei and Yudong Chen},
	title = {{Hidden Integrality and Semirandom Robustness of {SDP} Relaxation for Sub-Gaussian Mixture Model}},
	journal = {Mathematics of Operations Research},
	volume = {47},
	number = {3},
	pages = {2464--2493},
	year = {2022},
	doi = {10.1287/moor.2021.1216},
	url = {https://pubsonline.informs.org/doi/10.1287/moor.2021.1216}
}

@inproceedings{macqueen1967multivariate,
	title = {Some methods for classification and analysis of multivariate observations},
	author = {MacQueen, J.},
	booktitle = {Proceedings of the Fifth Berkeley Symposium on Mathematical Statistics and Probability},
	volume = {1},
	pages = {281--297},
	year = {1967},
	publisher = {University of California Press},
	address = {Berkeley and Los Angeles},
	url = {https://digicoll.lib.berkeley.edu/record/113015}
}

@article{lloyd1982least,
	title = {{Least squares quantization in PCM}},
	author = {Lloyd, Stuart P.},
	journal = {IEEE Transactions on Information Theory},
	volume = {28},
	number = {2},
	pages = {129--137},
	year = {1982},
	publisher = {IEEE},
	doi = {10.1109/TIT.1982.1056489},
	url = {https://doi.org/10.1109/TIT.1982.1056489}
}

@techreport{dasgupta2008hardness,
	title = {The hardness of k-means clustering},
	author = {Dasgupta, Sanjoy},
	year = {2008},
	institution = {University of California, San Diego},
	number = {CS2008-0916},
	url = {https://escholarship.org/uc/item/2qm3k10c}
}

@article{aloise2009np,
	title = {{NP}-hardness of {Euclidean} sum-of-squares clustering},
	author = {Aloise, Daniel and Deshpande, Amit and Hansen, Pierre and Popat, Preyas},
	journal = {Machine Learning},
	volume = {75},
	number = {2},
	pages = {245--248},
	year = {2009},
	publisher = {Springer},
	doi = {10.1007/s10994-009-5103-0},
	url = {https://link.springer.com/article/10.1007/s10994-009-5103-0}
}

@article{von2007tutorial,
	title = {A tutorial on spectral clustering},
	author = {von Luxburg, Ulrike},
	journal = {Statistics and Computing},
	volume = {17},
	number = {4},
	pages = {395--416},
	year = {2007},
	publisher = {Springer},
	doi = {10.1007/s11222-007-9033-z},
	url = {https://link.springer.com/article/10.1007/s11222-007-9033-z}
}

@inproceedings{ng2001spectral,
	title = {On spectral clustering: Analysis and an algorithm},
	author = {Ng, Andrew Y. and Jordan, Michael I. and Weiss, Yair},
	volume = {14},
	year = {2001},
	booktitle = {Advances in Neural Information Processing Systems},
	pages = {849--856},
	url = {https://proceedings.neurips.cc/paper/2001/hash/801272ee79cfde7fa5960571fee36b9b-Abstract.html}
}

@article{he2011symmetric,
	title = {Symmetric nonnegative matrix factorization: Algorithms and applications to probabilistic clustering},
	author = {He, Zhaoshui and Xie, Shengli and Zdunek, Rafal and Zhou, Guoxu and Cichocki, Andrzej},
	journal = {IEEE Transactions on Neural Networks},
	volume = {22},
	number = {12},
	pages = {2117--2131},
	year = {2011},
	publisher = {IEEE},
	doi = {10.1109/TNN.2011.2172457},
	url = {https://doi.org/10.1109/TNN.2011.2172457}
}

@article{kuang2015symnmf,
	title = {{SymNMF}: nonnegative low-rank approximation of a similarity matrix for graph clustering},
	author = {Kuang, Da and Yun, Sangwoon and Park, Haesun},
	journal = {Journal of Global Optimization},
	volume = {62},
	number = {3},
	pages = {545--574},
	year = {2015},
	publisher = {Springer},
	doi = {10.1007/s10898-014-0247-2},
	url = {https://link.springer.com/article/10.1007/s10898-014-0247-2}
}

@article{wang2012nonnegative,
	title = {Nonnegative matrix factorization: A comprehensive review},
	author = {Wang, Yu-Xiong and Zhang, Yu-Jin},
	journal = {IEEE Transactions on Knowledge and Data Engineering},
	volume = {25},
	number = {6},
	pages = {1336--1353},
	year = {2013},
	publisher = {IEEE},
	doi = {10.1109/TKDE.2012.51},
	url = {https://doi.org/10.1109/TKDE.2012.51}
}

@article{peng2007approximating,
	title = {Approximating {K}-means-type clustering via semidefinite programming},
	author = {Peng, Jiming and Wei, Yu},
	journal = {SIAM Journal on Optimization},
	volume = {18},
	number = {1},
	pages = {186--205},
	year = {2007},
	publisher = {SIAM},
	doi = {10.1137/050641983},
	url = {https://epubs.siam.org/doi/10.1137/050641983}
}

@article{mixon2017clustering,
	title = {Clustering subgaussian mixtures by semidefinite programming},
	author = {Mixon, Dustin G and Villar, Soledad and Ward, Rachel},
	journal = {Information and Inference: A Journal of the IMA},
	volume = {6},
	number = {4},
	pages = {389--415},
	year = {2017},
	publisher = {Oxford University Press},
	doi = {10.1093/imaiai/iax001},
	url = {https://academic.oup.com/imaiai/article/6/4/389/3074185}
}

@article{giraud2019partial,
	title = {Partial recovery bounds for clustering with the relaxed {$K$}-means},
	author = {Giraud, Christophe and Verzelen, Nicolas},
	journal = {Mathematical Statistics and Learning},
	volume = {1},
	number = {3/4},
	pages = {317--374},
	year = {2018},
	doi = {10.4171/MSL/8},
	url = {https://ems.press/journals/msl/articles/16225}
}

@inproceedings{zhuang2024statistically,
	title = {{Statistically optimal $K$-means clustering via nonnegative low-rank semidefinite programming}},
	author = {Zhuang, Yubo and Chen, Xiaohui and Yang, Yun and Zhang, Richard},
	booktitle = {International Conference on Learning Representations (ICLR)},
	volume = {2024},
	pages = {18825--18857},
	year = {2024},
	url = {https://proceedings.iclr.cc/paper_files/paper/2024/hash/5212d235f5d07a25170d4807d4d2d824-Abstract-Conference.html}
}

@inproceedings{xu2025scalablesecondorderriemannianoptimization,
	author = {Xu, Peng and Hou, Chun-Ying and Chen, Xiaohui and Zhang, Richard Y.},
	title = {{Scalable Second-order Riemannian Optimization for $K$-means Clustering}},
	year = {2026},
	booktitle = {International Conference on Learning Representations (ICLR)},
	url = {https://proceedings.iclr.cc/paper_files/paper/2026/hash/a7311af5d481dc9f6d3c993dbf8d2e21-Abstract-Conference.html}
}

@InProceedings{ZhuangChenYang2022LASDP,
	title = 	 {Likelihood Adjusted Semidefinite Programs for Clustering Heterogeneous Data},
	author =       {Zhuang, Yubo and Chen, Xiaohui and Yang, Yun},
	booktitle = 	 {Proceedings of the 40th International Conference on Machine Learning},
	pages = 	 {43326--43346},
	year = 	 {2023},
	volume = 	 {202},
	url = 	 {https://proceedings.mlr.press/v202/zhuang23a.html}
}

@article{chenyang2021threshold,
	author = {Chen, Xiaohui and Yang, Yun},
	doi = {10.1109/TIT.2021.3063155},
	issn = {1557-9654},
	journal = {IEEE Transactions on Information Theory},
	month = jun,
	number = {6},
	pages = {4223--4238},
	title = {{Cutoff for Exact Recovery of Gaussian Mixture Models}},
	volume = {67},
	year = {2021},
	url = {https://doi.org/10.1109/TIT.2021.3063155}
}

@ARTICLE{qianzhangchen2022KMeans,
	author = {Qian, Wei and Zhang, Yuqian and Chen, Yudong},
	journal = {IEEE Transactions on Information Theory},
	title = {Structures of Spurious Local Minima in {$k$}-Means},
	year = {2022},
	volume = {68},
	number = {1},
	pages = {395--422},
	doi = {10.1109/TIT.2021.3122465},
	url = {https://doi.org/10.1109/TIT.2021.3122465}
}

@article{yang2015sdpnal+,
	title = {{SDPNAL+}: a majorized semismooth {Newton-CG} augmented {Lagrangian} method for semidefinite programming with nonnegative constraints},
	author = {Yang, Liuqin and Sun, Defeng and Toh, Kim-Chuan},
	journal = {Mathematical Programming Computation},
	volume = {7},
	number = {3},
	pages = {331--366},
	year = {2015},
	publisher = {Springer},
	doi = {10.1007/s12532-015-0082-6},
	url = {https://link.springer.com/article/10.1007/s12532-015-0082-6}
}

@article{burer2003nonlinear,
	title = {A nonlinear programming algorithm for solving semidefinite programs via low-rank factorization},
	author = {Burer, Samuel and Monteiro, Renato D. C.},
	journal = {Mathematical Programming},
	volume = {95},
	number = {2},
	pages = {329--357},
	year = {2003},
	publisher = {Springer},
	doi = {10.1007/s10107-002-0352-8},
	url = {https://link.springer.com/article/10.1007/s10107-002-0352-8}
}

@inproceedings{kulis2007fast,
	title = {Fast low-rank semidefinite programming for embedding and clustering},
	author = {Kulis, Brian and Surendran, Arun C and Platt, John C},
	booktitle = {Proceedings of the Eleventh International Conference on Artificial Intelligence and Statistics},
	pages = {235--242},
	year = {2007},
	volume = {2},
	series = {Proceedings of Machine Learning Research},
	publisher = {PMLR},
	url = {https://proceedings.mlr.press/v2/kulis07a.html}
}

@book{NocedalWright2006_NumOpt,
	author = {Jorge Nocedal and Stephen J. Wright},
	title = {Numerical Optimization},
	series = {Springer Series in Operations Research and Financial Engineering},
	publisher = {Springer New York, NY},
	year = {2006},
	doi = {10.1007/978-0-387-40065-5},
	isbn = {978-0-387-40065-5},
	edition = {2nd},
	url = {https://link.springer.com/book/10.1007/978-0-387-40065-5}
}

@article{LEVINE2015184,
	title = {Data-Driven Phenotypic Dissection of {AML} Reveals Progenitor-like Cells that Correlate with Prognosis},
	journal = {Cell},
	volume = {162},
	number = {1},
	pages = {184--197},
	year = {2015},
	issn = {0092-8674},
	doi = {10.1016/j.cell.2015.05.047},
	url = {https://www.sciencedirect.com/science/article/pii/S0092867415006376},
	author = {Jacob H. Levine and Erin F. Simonds and Sean C. Bendall and Kara L. Davis and El-ad D. Amir and Michelle D. Tadmor and Oren Litvin and Harris G. Fienberg and Astraea Jager and Eli R. Zunder and Rachel Finck and Amanda L. Gedman and Ina Radtke and James R. Downing and Dana Pe’er and Garry P. Nolan}
}

@misc{CyTOFClean,
	title = {Clustering benchmark data: 32-dimensional data set from {Levine et al.} (2015)},
	author = {Lukas Weber},
	year = {2015},
	note = {GitHub repository; accessed September 25, 2026},
	url = {https://github.com/lmweber/benchmark-data-Levine-32-dim}
}

@article{GaiaDR3,
	author = {{Gaia Collaboration} and Vallenari, A. and Brown, A. G. A. and Prusti, T. and de Bruijne, J. H. J. and others},
	title = {{Gaia} Data Release 3: Summary of the content and survey properties},
	doi = "10.1051/0004-6361/202243940",
	journal = {Astronomy \& Astrophysics},
	year = 2023,
	volume = 674,
	pages = "A1",
	url = {https://doi.org/10.1051/0004-6361/202243940}
}

@article{GaiaXP,
	author = {De Angeli, F. and Weiler, M. and Montegriffo, P. and Evans, D. W. and Riello, M. and others},
	title = {{Gaia} Data Release 3: Processing and validation of {BP/RP} low-resolution spectral data},
	doi = "10.1051/0004-6361/202243680",
	journal = {Astronomy \& Astrophysics},
	year = 2023,
	volume = 674,
	pages = "A2",
	url = {https://doi.org/10.1051/0004-6361/202243680}
}

@article{GaiaPara,
	author = {Creevey, O. L. and Sordo, R. and Pailler, F. and Fr{\'e}mat, Y. and Heiter, U. and others},
	title = {{Gaia} Data Release 3: Astrophysical parameters inference system ({Apsis}). {I}. Methods and content overview},
	doi = "10.1051/0004-6361/202243688",
	journal = {Astronomy \& Astrophysics},
	year = 2023,
	volume = 674,
	pages = "A26",
	url = {https://doi.org/10.1051/0004-6361/202243688}
}

@inproceedings{CarsonMixonVillarWard_manifold-Kmeans,
	author = {Carson, Timothy and Mixon, Dustin G. and Villar, Soledad and Ward, Rachel},
	booktitle = {2017 International Conference on Sampling Theory and Applications (SampTA)},
	title = {Manifold optimization for $k$-means clustering},
	year = {2017},
	pages = {73--77},
	doi = {10.1109/SAMPTA.2017.8024388},
	url = {https://ieeexplore.ieee.org/document/8024388}
}

@misc{fastkmeans2025,
	author = {Benjamin Clavi{\'e} and Benjamin Warner},
	title = {{fastkmeans}: Accelerated {KMeans} Clustering in {PyTorch} and {Triton}},
	year = {2025},
	howpublished = {\url{https://github.com/AnswerDotAI/fastkmeans/}}
}

@article{hunt2023openclusters,
	author = {Hunt, Emily L. and Reffert, Sabine},
	title = {Improving the open cluster census. {II}. An all-sky cluster catalogue with {Gaia DR3}},
	journal = {Astronomy \& Astrophysics},
	volume = {673},
	pages = {A114},
	year = {2023},
	doi = {10.1051/0004-6361/202346285},
	url = {https://doi.org/10.1051/0004-6361/202346285}
}

@misc{omer2020fastpytorchkmeans,
	author = {Omer, Sehban},
	title = {{fast-pytorch-kmeans}},
	year = {2020},
	note = {GitHub repository; accessed September 25, 2026},
	url = {https://github.com/DeMoriarty/fast_pytorch_kmeans},
}

@article{halko2011,
	author = {Halko, N. and Martinsson, P. G. and Tropp, J. A.},
	title = {Finding Structure with Randomness: Probabilistic Algorithms for Constructing Approximate Matrix Decompositions},
	journal = {SIAM Review},
	volume = {53},
	number = {2},
	pages = {217-288},
	year = {2011},
	doi = {10.1137/090771806},
}

@misc{lu2016statistical,
	author = {Lu, Yu and Zhou, Harrison H.},
	title = {Statistical and Computational Guarantees of {Lloyd's} Algorithm and Its Variants},
	year = {2016},
	archivePrefix = {arXiv},
	eprint = {1612.02099},
	primaryClass = {math.ST},
	url = {https://arxiv.org/abs/1612.02099}
}

@misc{abbas2023semdedup,
	title = {{SemDeDup}: Data-efficient learning at web-scale through semantic deduplication},
	author = {Abbas, Amro and Tirumala, Kushal and Simig, D{\'a}niel and Ganguli, Surya and Morcos, Ari S},
	year = {2023},
	archivePrefix = {arXiv},
	eprint = {2303.09540},
	primaryClass = {cs.LG},
	url = {https://arxiv.org/abs/2303.09540}
}

@misc{yang2026flash,
	title = {{Flash-KMeans: Fast and Memory-Efficient Exact K-Means}},
	author = {Yang, Shuo and Xi, Haocheng and Zhao, Yilong and Li, Muyang and Fan, Xiaoze and Zhang, Jintao and Cai, Han and Lin, Yujun and Li, Xiuyu and Keutzer, Kurt and others},
	year = {2026},
	archivePrefix = {arXiv},
	eprint = {2603.09229},
	primaryClass = {cs.DC},
	url = {https://arxiv.org/abs/2603.09229}
}

\clearpage
\appendix
\section{Techinical Results}
\subsection{Proof of \texorpdfstring{\zcref{prop:smooth}}{Proposition 1}}\label{appx:lip}
Fix $y$, and set
\[
\hat{\bOne}_n\coloneqq\frac{\bOne_n}{\sqrt n},
\qquad
\bar{y}(U)\coloneqq y+\beta h(U).
\]
Then $\norm{\hat{\bOne}_n}_2=1$ and $h(U)=UU^\top\hat{\bOne}_n-\hat{\bOne}_n$. Since $A$ is
symmetric, differentiation gives
\[
\nabla f(U)
=2AU+\hat{\bOne}_n\bar{y}(U)^\top U+\bar{y}(U)\hat{\bOne}_n^\top U.
\]
Throughout, we use the mixed spectral/Frobenius inequality
\begin{equation}\label{eqn:mixed}
\norm{AB}_F\leq\norm{A}_2\norm{B}_F.
\end{equation}

Take $U,V$ in the Frobenius ball
$\{W:\norm{W}_F\leq\sqrt K\}$, and let
$\Delta\coloneqq U-V$. Using
\[
UU^\top-VV^\top=\Delta U^\top+V\Delta^\top
\]
and (\ref{eqn:mixed}), we obtain
\[
\norm{UU^\top-VV^\top}_F
\leq\bigl(\norm{U}_2+\norm{V}_2\bigr)\norm{\Delta}_F.
\]
Consequently,
\begin{equation}
\norm{h(U)-h(V)}_2=\norm{(UU^\top-VV^\top)e}_2\leq\bigl(\norm{U}_2+\norm{V}_2\bigr)\norm{\Delta}_F.
\label{eqn:g_bound}
\end{equation}
Moreover, throughout this ball,
\[
\norm{U}_2\leq\sqrt K,\qquad
\norm{h(U)}_2
\leq\norm{UU^\top}_2+\norm{e}_2
\leq K+1,
\]
so
\[
\norm{\bar{y}(U)}_2\leq\norm{y}_2+\beta(K+1).
\]

For the second term in the gradient, decompose
\[
\bar{y}(U)^\top U-\bar{y}(V)^\top V
=\bar{y}(U)^\top\Delta+
 [\bar{y}(U)-\bar{y}(V)]^\top V.
\]
Since $\norm{\hat{\bOne}_n}_2=1$, the mixed norm inequality and
(\ref{eqn:g_bound}) give
\begin{align*}
\norm{\hat{\bOne}_n\bar{y}(U)^\top U-\hat{\bOne}_n\bar{y}(V)^\top V}_F
&\leq \norm{\bar{y}(U)}_2\norm{\Delta}_F
 +\beta\norm{h(U)-h(V)}_2\norm{V}_2\\
&\leq \left(\norm{\bar{y}(U)}_2+
 \beta\bigl(\norm{U}_2+\norm{V}_2\bigr)\norm{V}_2\right)
 \norm{\Delta}_F.
\end{align*}
For the third term, use the separate decomposition
\[
\bar{y}(U)\hat{\bOne}_n^\top U-\bar{y}(V)\hat{\bOne}_n^\top V
=\bar{y}(U)\hat{\bOne}_n^\top\Delta+
 [\bar{y}(U)-\bar{y}(V)]\hat{\bOne}_n^\top V.
\]
The same estimates yield
\[
\norm{\bar{y}(U)\hat{\bOne}_n^\top U-\bar{y}(V)\hat{\bOne}_n^\top V}_F
\leq \left(\norm{\bar{y}(U)}_2+
 \beta\bigl(\norm{U}_2+\norm{V}_2\bigr)\norm{V}_2\right)
 \norm{\Delta}_F.
\]

Finally, using $\norm{A}_2=1$ and
$\norm{U}_2,\norm{V}_2\leq\sqrt K$, we have
\begin{align*}
\norm{\nabla f(U)-\nabla f(V)}_F
&\leq\left(2+2\norm{\bar{y}(U)}_2+4\beta K\right)\norm{\Delta}_F\\
&\leq\left(2+2\norm{y}_2+\beta(6K+2)\right)\norm{\Delta}_F.
\end{align*}
Thus the gradient is Lipschitz on the stated ball with constant
$\sfL$ in (\ref{eq:L}).

\subsection{Proof of \texorpdfstring{\zcref{cor:descent}}{Corollary 1}}
Now let $U\in\Omega$ and
$U^+\in\Proj_\Omega(U-\alpha\nabla f(U))$ be an exact Euclidean projection, and set $D\coloneqq U^+-U$. Since $U\in\Omega$, the projection property implies
\[
\norm{U^+-(U-\alpha\nabla f(U))}_F^2\leq\norm{U-(U-\alpha\nabla f(U))}_F^2.
\]
Expanding both sides gives
\[
\inner{\nabla f(U),D}\leq-\frac{1}{2\alpha}\norm{D}_F^2.
\]
Both $U$ and $U^+$ lie in the ball of radius $\sqrt K$, and that ball is convex. The smoothness inequality therefore gives
\[
f(U^+)\leq f(U)+\inner{\nabla f(U),D} +\frac{\sfL}{2}\norm{D}_F^2.
\]
Combining the last two inequalities proves (\ref{eq:descent}).

\subsection{Refining the Lipschitz constant}
The global bound in \zcref[S]{prop:smooth} can be sharpened locally near full feasibility using the following result.

\begin{lemma}\label{lem:opnorm}
For any $U\in\Omega$ satisfying $h(U)=0$, we have $\norm{U}_2=1$.
\end{lemma}
\begin{proof}
Let $M\coloneqq UU^\top$. Then $M$ is symmetric positive semidefinite and entrywise nonnegative, and $M\bOne_n=\bOne_n$. Hence $\norm{M}_\infty=\norm{M}_1=1$, so
\[
\norm{M}_2\leq
\sqrt{\norm{M}_1\norm{M}_\infty}=1.
\]
On the other hand, $M\bOne_n=\bOne_n$ implies that $1$ is an eigenvalue of $M$. Therefore,
$\norm{U}_2^2=\norm{M}_2=1$.
\end{proof}

\begin{remark}
\zcref{lem:opnorm} applies to every fully feasible factor, regardless of optimality. In particular, if $U^\ast(U^\ast)^\top=Z^\ast$ and $Z^\ast$ is the cluster membership matrix, then, after reordering the samples, $Z^\ast$ is block-diagonal with blocks $\bOne\bOne^\top/n_k$. Each block is the orthogonal projection onto $\Span(\bOne)$ and has spectrum contained in $\{0,1\}$.
\end{remark}

To quantify the improvement near feasibility, note that the differential of $h$ satisfies
\[
\mdif*{h(U)}[D]=(DU^\top+UD^\top)\hat{\bOne}_n,
\qquad
\norm{\mdif*{h(U)}[D]}_2
\leq 2\norm{U}_2\norm{D}_F.
\]
Differentiating the gradient therefore gives the bound
\[
\norm{\nabla^2 f(U)[D]}_F
\leq
\left(
2+2\norm{y+\beta h(U)}_2+4\beta\norm{U}_2^2
\right)\norm{D}_F.
\]
At any fully feasible $U$, \zcref{lem:opnorm} yields
\begin{equation}\label{eqn:R_feasible}
\norm{\nabla^2 f(U)[D]}_F
\leq
\left(2+2\norm{y}_2+4\beta\right)\norm{D}_F.
\end{equation}
Consequently, for fixed $y$ and $\beta$, and any $\varepsilon>0$, the gradient is Lipschitz with constant
$2+2\norm{y}_2+4\beta+\varepsilon$ on a sufficiently small ball around any fully feasible point. This gives a local improvement over the global bound in \zcref{prop:smooth}.

The bound has no explicit dependence on $n$ or $K$ apart from any dependence through $\norm{y}_2$ and $\beta$. We next control the optimal dual multiplier. Throughout this calculation, $A=-X X^\top$ denotes the objective matrix before normalization.

Following \citet[Assumption 2]{zhuang2024statistically}, consider the dual certificate constructed by \citet{chenyang2021threshold}:
\[
(y^\ast)^\top=\begin{pmatrix}
(y_{G_1^\ast}^\ast)^\top & \dots & (y_{G_K^\ast}^\ast)^\top
\end{pmatrix},
\]
where, writing $A_k\coloneqq A_{G_k^\ast G_k^\ast}$,
\[
y_{G_k^\ast}^\ast=-\frac{2}{n_k}A_k\bOne_{n_k}-\frac{\lambda}{n_k}\bOne_{n_k}+\frac{1}{n_k^2}\bOne_{n_k}\bigl(\bOne_{n_k}^\top A_k\bOne_{n_k}\bigr).
\]
Here $\lambda=d\sigma^2+mC\Theta^2/4$ for some $C\in(0,1)$,
and
$m=\min_{k\neq\ell}2n_kn_\ell/(n_k+n_\ell)$ is the minimum pairwise harmonic mean of the cluster sizes. For $i\in G_k^\ast$,
\[
[A_k\bOne_{n_k}]_i
=-\sum_{j\in G_k^\ast}X_i^\top X_j
=-n_kX_i^\top\bar X_k,
\qquad
\bOne_{n_k}^\top A_k\bOne_{n_k}
=-n_k^2\norm{\bar X_k}_2^2.
\]
Thus the certificate formula is equivalently
\[
y_i^\ast
=2X_i^\top\bar X_k-\norm{\bar X_k}_2^2
-\frac{\lambda}{n_k}.
\]

Under the normalization $s\coloneqq\norm{X}_2^2>0$, the objective matrix and the trace multiplier are divided by $s$. Since the equality constraint is additionally divided by $\sqrt n$, the corresponding normalized multiplier is
\(
\hat y^\ast\coloneqq(\sqrt n/s)y^\ast.
\)
We bound this multiplier on the event that the constructed multipliers satisfy dual feasibility and complementary slackness at the true cluster matrix $Z^\ast$.

First, for $K\geq2$, the certificate satisfies $0\leq\lambda\leq s$. Let $B\in\bbR^{n\times n}$ be the symmetric, entrywise nonnegative dual multiplier associated with the constraint $Z\geq0$. Complementary slackness gives $\inner{B,Z^\ast}=0$, which implies $B_{G_k^\ast G_k^\ast}=0$ for every $k$. Indeed, its dual slack matrix satisfies
\[
Q\coloneqq
\lambda I+\frac12\bigl(\bOne_n(y^\ast)^\top
+y^\ast\bOne_n^\top\bigr)+A-B\succeq0,
\qquad B\geq0,
\]
and equality of the primal and dual objectives gives
\[
\bOne_n^\top y^\ast=-\inner{A,Z^\ast}-K\lambda.
\]
We therefore obtain
\[
0\leq\hat{\bOne}_n^\top Q\hat{\bOne}_n
=-\inner{A,Z^\ast-\hat{\bOne}_n\hat{\bOne}_n^\top}
 -(K-1)\lambda-\hat{\bOne}_n^\top B\hat{\bOne}_n.
\]
Because $B$ is entrywise nonnegative and $Z^\ast-\hat{\bOne}_n\hat{\bOne}_n^\top$ is an orthogonal projection of rank $K-1$, while $0\preceq-A\preceq sI$,
\[
(K-1)\lambda
\leq-\inner{A,Z^\ast-\hat{\bOne}_n\hat{\bOne}_n^\top}
\leq(K-1)s.
\]
Together with the chosen nonnegative value of $\lambda$, this proves the claim.

For each cluster, define
\[
q_k\coloneqq\frac{\bOne_{n_k}}{\sqrt{n_k}},
\qquad
a_k\coloneqq q_k^\top A_kq_k.
\]
The block formula for the certificate becomes
\[
y_{G_k^\ast}^\ast
=-\frac{1}{\sqrt{n_k}}
 \bigl(2A_k+(\lambda-a_k)I\bigr)q_k.
\]
Since $-sI\preceq A_k\preceq0$, we have $-s\leq a_k\leq0$. Together with $0\leq\lambda\leq s$, this gives $0\leq\lambda-a_k\leq2s$, and hence
\[
-2sI
\preceq2A_k+(\lambda-a_k)I
\preceq2sI.
\]
Consequently,
\[
\norm{2A_k+(\lambda-a_k)I}_2\leq2s,
\qquad
\norm*{y_{G_k^\ast}^\ast}_2\leq\frac{2s}{\sqrt{n_k}}.
\]
Summing the squared bounds over the disjoint blocks yields
\[
\norm{\hat y^\ast}_2
=\frac{\sqrt n}{s}\norm{y^\ast}_2
\leq2\sqrt{n\sum_{k=1}^K\frac{1}{n_k}}.
\]

Consequently, if the cluster sizes are balanced in the sense that $n_k\geq cn/K$ for a constant $c>0$, then
\[
\norm{\hat y^\ast}_2\leq\frac{2K}{\sqrt c}.
\]
In particular, for equal cluster sizes, $\norm{\hat y^\ast}_2\leq2K$. Thus the limiting local smoothness bound at the normalized dual certificate satisfies
\[
2+2\norm{\hat y^\ast}_2+4\beta
\leq
2+4\sqrt{n\sum_{k=1}^K\frac{1}{n_k}}+4\beta.
\]
For balanced clusters this is $O(K+\beta)$, with an implied constant depending only on $c$ and no additional dependence on the sample size, dimension, noise level, or separation.

This conclusion concerns the optimal dual certificate. To extend it to the dual iterates, one additionally needs control of $\norm{y_t-\hat y^\ast}_2$, since
\[
\norm{y_t}_2
\leq\norm{\hat y^\ast}_2
+\norm{y_t-\hat y^\ast}_2.
\]

\section{Pseudocode}
\begin{algorithm}[H]
    \caption{\GEM: detailed GPU implementation}
    \label{alg:fast_nlr}
    \begin{algorithmic}[1]
        \Require Row-major $\hat X\in\bbR^{n\times d}$,
        $U\in\Omega$ on the GPU;\\
        $\alpha,\beta>0$, $T,T_{\min}$, $\epsilon_d,\epsilon_c$.
        \Statex
        \State $\tilde y\gets0$; $\gamma\gets\beta/n$;
        \Comment{Stored dual: $\tilde y=y/\sqrt n$}
        \State $B\gets256$; $G\gets\lceil n/B\rceil$;
        \Comment{CUDA block and grid sizes}
        \State $N_{\rm tile}\gets\lceil n/128\rceil\lceil r/64\rceil$;
        \Comment{CUTLASS output tiles}

        \LComment{Initial constraint residual}
        \State $\Sgemv(\flagN,r,n,1,U,r,\bOne,1,0,p,1)$;
        \Comment{cuBLAS: $p=U^\top\bOne$}
        \State $\Sgemv(\flagT,r,n,1,U,r,p,1,0,q_r,1)$;
        \Comment{cuBLAS: $q_r=Up$}
        \State $\Saxpy(n,-1,\bOne,1,q_r,1)$;
        \Comment{cuBLAS: $q_r\gets q_r-\bOne$}

        \For{$t=0,\dotsc,T-1$}
            \LComment{(a) Objective-gradient intermediate}
            \State $\Sgemm(\flagN,\flagT,r,d,n,
                1,U,r,\hat X,d,0,\Theta,r)$;
            \Comment{cuBLAS: $\Theta=\hat X^\top U$}

            \LComment{(b) Constraint-gradient vectors}
            \State $\mathtt{compute\_ybar<<<G, B>>>}$%
                $(\bar y,\tilde y,q_r,\gamma,n)$;
            \Comment{CUDA: $\bar y=\tilde y+\gamma q_r$}
            \State $\Sgemv(\flagN,r,n,1,U,r,\bar y,1,0,s,1)$;
            \Comment{cuBLAS: $s=U^\top\bar y$}

            \LComment{(c) CUTLASS fused update and statistics}
            \State $\mathtt{FusedGemm}(\hat X,\Theta;
                U,\bar y,p,s,\alpha,
                a_{\rm tile},b_{\rm tile},h_{\rm tile})$;
            \Comment{CUTLASS; schematic arguments}
            \LComment{\hspace{\algorithmicindent} Epilogue:
                $u\gets U_{ij}$,
                $w\gets\max\{u+2\alpha(\hat X\Theta)_{ij}
                -\alpha\bar y_i p_j-\alpha s_j,0\}$}
            \LComment{\hspace{\algorithmicindent} FP64 accumulation:
                $w^2$, $u^2$, $wu$ over valid entries}
            \LComment{\hspace{\algorithmicindent} In-place store:
                $U_{ij}\gets w$; shared-memory reduction to tile partials}

            \LComment{(d) Global reduction and projection}
            \State $\mathtt{reduce\_vplus\_partials<<<1, B>>>}$%
                $(a_{\rm tile},b_{\rm tile},h_{\rm tile},
                a_{\rm dev},b_{\rm dev},h_{\rm dev},N_{\rm tile})$;
            \LComment{\hspace{\algorithmicindent} Per-thread strided FP64 sums;
                shared-memory block reduction}
            \LComment{\hspace{\algorithmicindent}
                $a = a_\text{dev}=\sum_\tau a_\tau$,
                $b = b_\text{dev}=\sum_\tau b_\tau$,
                $h = h_\text{dev}=\sum_\tau h_\tau$}
            \State $\nu\gets\sqrt a$;
            \Comment{Host: pre-scaling norm}
            \If{$(\nu<10^{-20})$}
                \State \textbf{break};
                \Comment{Normalization breakdown}
            \EndIf
            \State $c_f\gets\operatorname{FP32}(\sqrt K/\nu)$;
            \Comment{Host: FP64 ratio}
            \State $\mathtt{cublasSscal\_64}(nr,c_f,U,1)$;
            \Comment{cuBLAS: in-place normalization}

            \LComment{(e) Movement estimate on the host}
            \State $c\gets\operatorname{FP64}(c_f)$;
            \Comment{Rounded scaling multiplier}
            \State $\Delta\gets\max\{c^2a+b-2ch,0\}$;
            \Comment{Squared movement estimate}
            \State $\hat\delta\gets\sqrt{\Delta}/\sqrt K$;
            \Comment{Relative movement}
            \State $\mathtt{dual}\gets(\Delta<\epsilon_d^2K)$;
            \Comment{Dual-update gate}

            \LComment{(f) Updated constraint residual}
            \State $\Sgemv(\flagN,r,n,1,U,r,\bOne,1,0,p,1)$;
            \Comment{cuBLAS: $p=U^\top\bOne$}
            \State $\Sgemv(\flagT,r,n,1,U,r,p,1,0,q_r,1)$;
            \Comment{cuBLAS: $q_r=Up$}
            \State $\Saxpy(n,-1,\bOne,1,q_r,1)$;
            \Comment{cuBLAS: $q_r\gets q_r-\bOne$}

            \LComment{(g) Conditional dual ascent and stopping}
            \If{$\mathtt{dual}$}
                \State $\Saxpy(n,\gamma,q_r,1,\tilde y,1)$;
                \Comment{cuBLAS: $\tilde y\gets\tilde y+\gamma q_r$}
            \EndIf
            \State $\SnrmTwo(n,q_r,1,\mathtt{constraint\_norm})$;
            \Comment{cuBLAS: host result $\|q_r\|_2$}
            \State $\rho\gets\mathtt{constraint\_norm}/\sqrt n$;
            \Comment{Host: normalized feasibility}
            \If{$\mathtt{dual}$ \textbf{and}
                $\hat\delta<\epsilon_c$ \textbf{and}
                $\rho<\epsilon_c$ \textbf{and}
                $t+1\geq T_{\min}$}
                \State \textbf{break}.
            \EndIf
        \EndFor
    \end{algorithmic}
\end{algorithm}

\clearpage
\zcref[S]{alg:nlr} shows the original CPU implementation of the NLR algorithm~\citep{zhuang2024statistically}. We collect it here for reference.
\begin{algorithm}[H]
\caption{Hardware-agnostic NLR $K$-means algorithm}\label{alg:nlr}
\label{alg:vanilla_NLR}
\begin{algorithmic}[1]
\Require Dissimilarity matrix $A=-X^\top X$;\\
         cluster parameter $K>0$, rank parameter $r\geq K$;\\
         step size $\alpha > 0$, augmentation parameter $\beta > 0$\\
         primal tolerance $\varepsilon_p$, tolerance $\varepsilon_d$.
\Statex
\For{$t=1, \dotsc, T$}
	\State $U_{t, 0}\gets U_t$;
	\LComment{Primal descent}
    \For{$s=1, \dotsc, S$}
        \State $\bar{y}\gets y+\beta(U_{t, s-1}U_{t, s-1}^\top\bOne-\bOne)$;
        \State $\nabla\mathcal{L}\gets (2A+2\ell I+\bOne_n\bar{y}^\top+\bar{y}\bOne^\top)U_{t, s-1}$;
        \State $U_{t, s}=\Proj(U_{t, s-1}-\alpha \nabla\mathcal{L})$, where $\Proj(V)=\sqrt{K}V_+/\norm{V_+}_F$;
        \Comment{Projected gradient descent}
    \EndFor
    \State $U_{t+1}\gets U_{t, S}$;
    \LComment{Dual ascent}
    \If{$\norm{U_{t+1}-U_t}_F < \varepsilon_p$}
        \State $\bm{s}\gets U_{t+1}U_{t+1}^\top\bOne-\bOne$;
        \LComment{Stopping criterion}
        \If{$\max\{\norm{U_{t+1}-U_t}_F, \norm{\bm{s}}_2\}<\varepsilon_d$}
	       \State \textbf{break};
        \EndIf
        \State $y\gets y+\beta\bm{s}$.
    \EndIf
\EndFor
\end{algorithmic}
\end{algorithm}

\clearpage

\section{GPU Implementation Details}
\label{app:gpuops}

This appendix maps the GPU calls in \zcref[S]{alg:fast_nlr} to the normalized iteration. The CUTLASS epilogue, movement estimate, and storage conventions are described in the main text.

\subsection{Storage and parameter conventions}

The matrices $\hat X$, $U$, and $\Theta$ are stored in row-major order. cuBLAS interprets these buffers as column-major matrices with transposed dimensions:
\begin{equation}
 \begin{array}{c|c|c}
 \text{Row-major matrix}
 & \text{cuBLAS interpretation}
 & \text{Leading dimension}\\
 \hline
 \hat X\in\bbR^{n\times d}
 & \hat X^\top\in\bbR^{d\times n} & d\\
 U\in\bbR^{n\times r}
 & U^\top\in\bbR^{r\times n} & r\\
 \Theta\in\bbR^{d\times r}
 & \Theta^\top\in\bbR^{r\times d} & r
 \end{array}
 \label{eq:layout}
\end{equation}
No explicit transpose is needed. CUTLASS reads the same buffers directly in row major order.

The mathematical dual variable $y$ multiplies the normalized residual $g(U)/\sqrt n$. The device buffer named $\mathtt{y}$ stores $y/\sqrt n$, and the scalar $\mathtt{penalty\_beta}$ equals $\beta/n$. Consequently, $\mathtt{compute\_ybar}$ evaluates
\[
 \bar y
 =\mathtt{y}+\mathtt{penalty\_beta}\,q_r
 =\frac{y}{\sqrt n}+\frac{\beta}{n}q_r.
\]
The conditional dual $\mathtt{Saxpy}$ adds $(\beta/n)q_r$ to this buffer, implementing $y\gets y+(\beta/\sqrt n)q_r$ in mathematical notation.

\paragraph{CUDA launch conventions.}
In \zcref{alg:fast_nlr}, $B=256$ is the thread-block size for the custom CUDA kernels and $G=\lceil n/B\rceil$ is the grid size for $\mathtt{compute\_ybar}$. The kernel $\mathtt{reduce\_vplus\_partials}$ uses one block of $B$ threads. CUTLASS configuration is described in \zcref{app:cutlass}.

\subsection{cuBLAS routine reference}

\zcref{tab:routines} summarizes the cuBLAS routines used in the iteration. The prefix \texttt{S} denotes single-precision real arithmetic. In this table, $\lambda$ and $\mu$ are generic BLAS coefficients, independent of the algorithmic parameters $\alpha$ and $\beta$.

\begin{table}[!htbp]
\centering
\small
\caption{cuBLAS routines used in the iteration. Matrices and vectors reside on the device; the norm result is returned to host memory in this implementation.}
\label{tab:routines}
\begin{tabular}{@{}lll@{}}
\toprule
Routine & BLAS level & Operation\\
\midrule
$\Sgemm$ & 3 &
$C\gets\lambda\op(A)\op(B)+\mu C$\\
$\Sgemv$ & 2 &
$v\gets\lambda\op(A)u+\mu v$\\
$\Saxpy$ & 1 &
$v\gets\lambda u+v$\\
$\SnrmTwo$ & 1 &
$\rho\gets\norm{u}_2$\\
$\Sscal$ & 1 &
$u\gets\lambda u$\\
\bottomrule
\end{tabular}
\end{table}

BLAS levels 1, 2, and 3 denote vector operations, matrix--vector operations, and matrix--matrix operations, respectively. The transpose flags specify
\[
 \op(A)=
 \begin{cases}
 A, & \flagN\;(\mathtt{CUBLAS\_OP\_N}),\\
 A^\top, & \flagT\;(\mathtt{CUBLAS\_OP\_T}).
 \end{cases}
\]
Setting $\mu=0$ overwrites the output without using its previous contents.

In $\Sgemm$, $(m,n,k)$ specify the product dimensions: $\op(A)$ is $m\times k$, $\op(B)$ is $k\times n$, and $C$ is $m\times n$. In $\Sgemv$, $(m,n)$ specify the stored matrix $A$ before applying the transpose flag. Leading dimensions refer to the stored arrays.

\paragraph{GPU calls and execution order.} \zcref[S]{tab:callsites-memory} lists the computational calls in execution order. Before the loop, calls 7--9 are also used to initialize $p$ and $q_r$.

\begin{table}[!htbp]
\centering
\footnotesize
\setlength{\tabcolsep}{2pt}
\caption{GPU kernel call sites for one normalized GEM iteration. The first column gives kernel order, not pseudocode line numbers. For cuBLAS matrix calls, flags and dimensions are listed as passed to the library.}
\label{tab:callsites-memory}
\begin{tabular}{@{}r p{0.25\textwidth}
                    p{0.24\textwidth}
                    p{0.44\textwidth}@{}}
\toprule
 & Routine & Configuration & Mathematical interpretation\\
\midrule
1 & $\Sgemm$
  & $\flagN,\flagT;\ r,d,n$
  & $\Theta^\top\gets U^\top\hat X$;
    the output buffer stores $\Theta=\hat X^\top U$.\\[4pt]

2 & $\mathtt{compute\_ybar}$
  & Elementwise CUDA kernel
  & $\bar y\gets y/\sqrt n+(\beta/n)q_r$.\\[4pt]

3 & $\Sgemv$
  & $\flagN;\ r,n$
  & $s\gets U^\top\bar y$.\\[4pt]

4 & CUTLASS $\mathtt{FusedGemm}$
  & Row-major operands and problem shape $(n, r, d)$.
  & Compute $W$ from
    (\ref{eq:gpu-positive-update}), accumulate tile
    partials of $a,b,h$, and overwrite $U$ with $W$.\\[4pt]

5 & $\mathtt{reduce\_vplus\_partials}$
  & One block of $B=256$ threads
  & Reduce the three arrays of tile partials to
    $a,b,h$ in (\ref{eq:fused-statistics}).\\[4pt]

6 & $\Sscal$
  & Length $nr$; multiplier $c$
  & Scale the overwritten factor buffer:
    $U\gets cW$.\\[4pt]

7 & $\Sgemv$
  & $\flagN;\ r,n$
  & $p\gets U^\top\bOne$.\\[4pt]

8 & $\Sgemv$
  & $\flagT;\ r,n$
  & $q_r\gets Up$.\\[4pt]

9 & $\Saxpy$
  & Length $n$; multiplier $-1$
  & $q_r\gets q_r-\bOne$.\\[4pt]

10 & $\Saxpy$ (conditional)
   & Length $n$; multiplier $\beta/n$
   & Update the stored dual buffer by
     $(\beta/n)q_r$, equivalent to
     $y\gets y+(\beta/\sqrt n)q_r$.\\[4pt]

11 & $\SnrmTwo$
   & Length $n$
   & Return $\norm{q_r}_2$ for the feasibility
     diagnostic.\\
\bottomrule
\end{tabular}
\end{table}

\subsection{CUTLASS configuration and epilogue construction}
\label{app:cutlass}

The second matrix product, $\hat X\Theta$, uses a CUTLASS GEMM with row-major operands and problem shape $(n,r,d)$. The kernel is instantiated through \texttt{DefaultGemmWithVisitor} and exposed through \texttt{GemmUniversal\allowbreak{}Adapter}. It uses the \texttt{Sm80} architecture tag with \texttt{OpClassTensorOp}, threadblock tile $(128,64,16)$, warp tile $(64,32,16)$, instruction shape $(16,8,8)$, and three mainloop stages. Storage and accumulation are FP32; the selected Tensor Core multiplication uses TF32 internally. The operand alignment is one element, and the epilogue uses one stage.

\paragraph{Epilogue visitor tree.}
The epilogue is constructed recursively using \texttt{Sm80EVT<Node, Children...>}. Each parent receives fragments from its children and applies its operation. \zcref{fig:epilogue-tree} shows the resulting composition. Let $a_{ij}=(\hat X\Theta)_{ij}$ denote the GEMM accumulator.

\begin{figure}[!htbp]
\centering
\begin{forest}
for tree={
    grow'=0,
    draw,
    rounded corners,
    font=\scriptsize,
    align=center,
    parent anchor=east,
    child anchor=west,
    anchor=west,
    l sep=8pt,
    s sep=6pt,
    inner sep=3pt,
    calign=midpoint,
    tier/.option=level,
    edge={<-}
}
[{\texttt{UStore}\\$U_{ij}\gets w$}
    [{\texttt{VPlusStats}\\FP64 sums: $w^2,u^2,wu$}
        [{Maximum\\$w=\max\{v,0\}$}
            [{Add}
                [{Add}
                    [{Add}
                        [{\texttt{ULoad}\\$u=U_{ij}$}]
                        [{Multiply}
                            [{Accumulator\\$a_{ij}$}]
                            [{Scalar\\$2\alpha$}]
                        ]
                    ]
                    [{Multiply}
                        [{Multiply}
                            [{Column broadcast\\$\bar y_i$}]
                            [{Row broadcast\\$p_j$}]
                        ]
                        [{Scalar\\$-\alpha$}]
                    ]
                ]
                [{Multiply}
                    [{Row broadcast\\$s_j$}]
                    [{Scalar\\$-\alpha$}]
                ]
            ]
            [{Scalar\\$0$}]
        ]
        [{\texttt{ULoad}\\$u=U_{ij}$}]
    ]
]
\end{forest}
\caption{CUTLASS epilogue visitor tree. Arrows indicate fragment flow from children to parents,
from the leaves on the right toward the store on the left. Arithmetic nodes form the nonnegative update, the custom visitor accumulates statistics, and the root stores the updated factor entry.}
\label{fig:epilogue-tree}
\end{figure}

\paragraph{Loads and arithmetic.}
The accumulator leaf supplies $a_{ij}$, and the auxiliary loads supply the old factor entry $U_{ij}$. Row broadcasts provide $p_j$ and $s_j$, while the column broadcast provides $\bar y_i$. Together with the scalar coefficients, these inputs evaluate the entrywise update in (\ref{eq:entrywise-fusion}). All nodes operate within the same GEMM kernel.

\paragraph{Statistics and storage.}
The custom statistics node receives the unnormalized update $w$ and the old entry $u$ before $U_{ij}$ is overwritten. It accumulates $w^2$, $u^2$, and $wu$ in FP64 over valid output entries and passes $w$ to the store. The two \texttt{ULoad} leaves represent two uses of the old factor, not separate factor buffers. Thread local sums are reduced to one partial per statistic per output tile. Subseqeunt global reductions and scaling are performed in separate calls outside the tree.

\section{Additional Computation Details}

\paragraph{Initialization.} The original NLR algorithm initializes $U_0=\Proj_\Omega(Z)$ with i.i.d.~$N(0,1)$ entries in $Z$. To incorporate data structure, we instead use
\[V=XX^\top Z,\qquad V=QR,\qquad U_0=\Proj_\Omega(Q).\] The first two steps apply the randomized range finder of \citet{halko2011} to the Gram matrix $XX^\top$. Multiplication by $XX^\top$ emphasizes its leading eigen-directions,
while QR decomposition produces an orthonormal basis for the sampled range. Although the subsequent projection need not preserve this range, the construction provides a data-informed initialization. Empirically, orthonormalizing before projection substantially reduces the initial augmented-Lagrangian value.

For efficient GPU computation, we implement the orthonormalization using Cholesky QR in the smaller feature space. Writing $M=X^\top Z$ and $C=X^\top X$, we have \[V=XM,\qquad V^\top V=M^\top C M=R^\top R,\qquad Q=X(MR^{-1}),\] where $R$ is the upper-triangular Cholesky factor and $MR^{-1}$ is computed by a triangular solve. This formulation replaces QR of an $n\times r$ matrix with an $r\times r$ Cholesky factorization and a small triangular solve, while the large operations remain GPU-efficient matrix multiplications. Moreover, $C$ can be reused from data normalization. That said, Cholesky QR requires a sufficiently well-conditioned, full-rank sample matrix; otherwise, a more stable orthogonalization is needed.

\paragraph{Recover cluster labels from $U$ matrix.} Under exact recovery, the optimal clustering matrix $U^\ast(U^\ast)^\top$ has the block structure of the true partition, and the rows of any nonnegative factor $U^\ast$ are constant within each cluster, see
\citet[Theorem~2]{zhuang2024statistically} and
\citet[Lemma~1]{xu2025scalablesecondorderriemannianoptimization}.
For an approximate solution $U$, we apply SVD and take its leading $K$ left singular
vectors as a spectral embedding, and then apply $K$-means to their rows to obtain cluster labels. Although this SVD becomes costly for large $U$, its runtime is small relative to the main optimization loop in our experiments.

\paragraph{Hardware information.}
All experiments were conducted on NVIDIA H200 GPU nodes equipped with Intel Xeon Platinum 8558 processors and 32 GB of RAM. Additionally, some of the code were also tested on NVIDIA A100 and GeForce RTX 5090 GPUs.

\end{document}